\documentclass{article} 
\usepackage[margin=1in]{geometry}

\usepackage{xspace} 
\usepackage{microtype} 
\usepackage{graphicx}
\usepackage{subfigure}
\usepackage{booktabs} 
\usepackage{amsfonts} 
\usepackage{amsmath,amsthm} 
\usepackage{amssymb,amsmath,amsthm}
\newtheorem{theorem}{Theorem}
\newtheorem{lemma}{Lemma}
\usepackage{mathtools} 

\usepackage{algorithm}
\usepackage{algorithmic} 
\usepackage{float} 
\restylefloat{table}
\usepackage{caption} 

\newcommand{\bbR}{\mathbb{R}} 
\newcommand{\bbE}{\mathbb{E}} 
\usepackage{stfloats} 
\usepackage{threeparttable} 
\usepackage{comment} 
\usepackage{xcolor} 
\usepackage{lmodern} 
\usepackage{wrapfig} 
\usepackage{enumitem}
\usepackage{caption} 
\usepackage{wrapfig} 
\usepackage{soul}

\allowdisplaybreaks 

\newcommand{\Mada}{\texttt{MADA}\xspace}
\newtheorem*{claim*}{Claim} 
\newtheorem*{corollary*}{Corollary}
\newtheorem{proposition}{Proposition}

\newtheorem{fact}{Fact}
\usepackage[colorlinks=true,citecolor=blue, linkcolor=blue]{hyperref}
\usepackage{natbib}
\usepackage{capt-of}
\usepackage[]{graphicx}
\usepackage[capitalize,noabbrev]{cleveref}

\graphicspath{{./figs/}}

\usepackage[textsize=tiny]{todonotes}

\newcommand{\kaancr}[1]{{{#1}}}

\newcommand{\kaan}[1]{{{#1}}}
\newcommand{\kaans}[1]{{{#1}}}

\title{\Mada: Meta-Adaptive Optimizers through hyper-gradient Descent\footnotetext{The work of KO was done while interning at Amazon Web Services. SS, MH, and VC hold concurrent appointments as an Amazon Scholar and as a faculty at Technion, University of Minnesota, and EPFL, respectively. This paper describes their work performed at Amazon.}}

\author{Kaan Ozkara\thanks{Department of Electrical and Computer Engineering, University of California Los Angeles\ \{kaan@ucla.edu\}} 
\and 
        Can Karakus\thanks{Amazon Web Services \ \{cakarak@amazon.com, prraman@amazon.com, bkveton@amazon.com, volkcevh@amazon.com\}}
\and 
        Parameswaran Raman$^\dag$
\and 
        Mingyi Hong\thanks{Department of Electrical and Computer Engineering, University of Minnesota \ \{mhong@umn.edu\}}        
\and 
        Shoham Sabach\thanks{Faculty of Data and Decision Sciences, Technion -- Israel Institute of Technology \ \{ssabach@technion.ac.il\}}
\and
        Branislav Kveton$^\dag$
\and    
        Volkan Cevher$^{\dag,}$\thanks{LIONS, IEM, STI, Ecole Polytechnique Fédérale de Lausanne \ \{volkan.cevher@epfl.ch\}}
}

\begin{document}
\date{} 
\maketitle

\begin{abstract}
Following the introduction of Adam, several novel adaptive optimizers for deep learning have been proposed. These optimizers typically excel in some tasks but may not outperform Adam uniformly across all tasks. In this work, we introduce Meta-Adaptive Optimizers (\Mada), a unified optimizer framework that can generalize several known optimizers and dynamically learn the most suitable one during training. The key idea in \Mada is to parameterize the space of optimizers and dynamically search through it using hyper-gradient descent during training. We empirically compare \Mada to other popular optimizers on vision and language tasks, and find that \Mada consistently outperforms Adam and other popular optimizers, and is robust against sub-optimally tuned hyper-parameters. \Mada achieves a greater validation performance improvement over Adam compared to other popular optimizers during GPT-2 training and fine-tuning. We also propose AVGrad, a modification of AMSGrad that replaces the maximum operator with averaging, which is more suitable for hyper-gradient optimization. Finally, we provide a convergence analysis to show that parameterized interpolations of optimizers can improve their error bounds (up to constants), hinting at an advantage for meta-optimizers.
\end{abstract}


\section{Introduction}
The choice of an optimization algorithm plays a critical role in  determining the downstream performance of a machine learning model. Adaptive moment optimizers are the most preferred class of optimizers employed in most learning tasks such as training Large Language Models (LLMs) \citep{brown2020language, touvron2023llama} and Diffusion Models \citep{rombach2022high}. In particular, Adam \citep{KingmaB14} is still the ``go-to'' optimizer in LLM training, despite the emergence of many other optimizers since then.

Recently proposed optimizers \citep{chen2023symbolic, xie2023adan, liu2023sophia, You2020Lamb, foret2021sam} report improved performance compared to Adam in specific tasks. However, it is  unclear if their strong performance generalizes across a wide range of tasks as Adam's does \citep{schmidt21a}. It is also unclear if a single optimizer can be uniformly the best across all learning tasks and training regimes, such as different batch sizes, hyper-parameters, and datasets.

In this work, we introduce the concept of a \emph{parameterized optimizer}, which can be viewed as a unified parameterization of a collection of given optimizers.
\kaan{The parameters of a parameterized optimizer define a convex polytope (e.g. hypercube), whose vertices may correspond to individual base optimizers, while the interior represents new optimizers formed by interpolating between them. 
}
To make this concept operational, we propose the meta-adaptive optimizer (\Mada), which combines the parameterized optimizer with hyper-gradient descent \citep{baydin2018online} to learn the optimizer coefficients. \Mada dynamically adjusts the parameterized optimizer coefficients \emph{during} training, and effectively adapts the optimizer choice to the learning task. While \Mada bears some connections to optimizer search methods, such as \citep{chen2023symbolic}, it does not solely output a final optimizer state. It dynamically adapts the optimizer to the current neighborhood of the loss landscape on-the-fly, removing the need for an offline optimizer search stage, or an outer hyper-parameter selection loop.


\textbf{Contributions.} We make the following contributions:
\begin{itemize}
    \itemsep0em
    \item We introduce the concept of a parameterized optimizer, which takes a collection of existing optimizers, and unifies them into a single optimizer, combining the individual update rules through learnable coefficients.
    \item We propose \Mada, a meta-optimization framework that combines parameterized optimizers with hyper-gradient descent to learn a specific optimizer instance during training.
    \item \kaan{We find that not all optimizers are suitable to use in a hyper-gradient optimization framework. Specifically, among popular optimizers, using AMSGrad \citep{reddi2018} results in poor performance within \Mada due to its use of the maximum operator. Motivated by this, we propose a modification of it called AVGrad, which replaces the maximum operator on the second-order moments with time-averaging, and leads to better performance when used as part of \Mada. We also analyze its convergence properties (see Appendix~\ref{app:avgrad}). We show that \Mada converges to AVGrad in a simple example where it is known to perform better than Adam, providing evidence for the effectiveness of \Mada in adapting the optimizer to the task (see Appendix~\ref{app:experiments}).}
    \item \kaan{To demonstrate that optimizer interpolations can improve convergence bounds in an analytically tractable scenario, we theoretically analyze the convergence behavior for interpolations between two optimizers, namely AVGrad and Adam. Our analysis shows that the interpolated optimizer improves the convergence bounds of the base optimizers up to constant factors. }
    \item We develop a specific parameterized optimizer, which interpolates between Adam \citep{KingmaB14}, AVGrad, Yogi \citep{zaheer2018adaptive}, Adan \citep{xie2023adan}, and Lion \citep{chen2023symbolic}. On language tasks we compare \Mada, which is based on this parameterized optimizer, against Adam, Lion, Adan, and HyperAdam\footnote{Throughout this paper we will refer to the version of Adam that uses hyper-gradients to tune $\beta_1$ and $\beta_2$ parameters, as in \cite{chandra2022gradient}, as HyperAdam.}, and show that \Mada consistently outperforms all baselines. We also illustrate the robustness of \Mada to initial hyper-parameters and analyze the evolution of hyper-parameters during training. \kaans{On vision tasks we compare \Mada to Adam, SGD with momentum, and HyperAdam, and observe consistent performance improvement.}
\end{itemize}
 
\textbf{Related work.}
Motivated by the high cost of training large language models, a large number of novel optimizers have been proposed in recent years to speed up the training, increase generalization performance or train more resource-efficient models. As mentioned before, Adam \citep{KingmaB14} is the most commonly employed optimizer, and succeeding methods, in general, try to improve upon it under different scenarios. \cite{reddi2018,zaheer2018adaptive} propose AMSGrad and YOGI, respectively, to fix Adam's potentially non-decreasing effective learning rate. \cite{xie2023adan,dozat2016nadam} introduce Adan and Nadam, respectively, to replace heavy-ball momentum in Adam with Nesterov momentum. \cite{you2017lars, You2020Lamb} propose LARS and LAMB to improve the performance of Adam in large-batch regime. \cite{heo2021adamp} proposes AdamP to avoid premature decay of scale-invariant weights. AdaBound \citep{luo2018adabound} and AdaBelief \citep{zhuang2020adabelief} try to estimate a more stable second-order moment term. \cite{foret2021sam} proposes sharpness-aware minimization (SAM) and \cite{chen2020padam} proposes Padam to increase the generalization performance of the trained models. \cite{Liu2020Oradam} stabilizes the training by reducing gradient variance. The work in \citep{chen2023symbolic} is similar in spirit to our work, in that it introduces a method to symbolically search a space of optimizers; however unlike our work they consider an offline search method, whereas our method learns the optimizer \emph{during} actual model training, and uses hyper-gradients. 

\kaancr{A separate line of work proposed learned optimizers to delegate the optimization task to neural networks (fully connected or LSTMs) \citep{almeida2021generalizable,metz2022velo,metz2020tasks}. Instead of directly using first order information as in optimizers, these methods treat gradient information, alongside other features such as training loss, validation loss \citep{almeida2021generalizable} and so on, as inputs to a neural network which outputs the update to be applied on a particular weight. Learned optimizers require resource-intensive offline training (e.g. thousands of TPU-months in \citep{metz2022velo}) since the updates are handled through a separate neural network that needs to be trained before deployment. Integration of a separate neural network to the optimization process introduces additional complexity, which \Mada avoids by requiring only a few parameters to be learned on-the-fly. }

Another related line of work is on gradient-based hyper-parameter optimization \citep{almeida1999,maclaurin2015gradient, franceschi2017forward, baydin2018online, chandra2022gradient}; our work applies a similar idea in a new setting, namely optimizer search and adaptation. Finally, we note that our work more broadly relates to a long line of research on AutoML and meta-learning \citep{andrychowicz2016learning, wichrowska2017learned, hospedales2021meta, real2020automlzero}.


\section{Parameterized Optimizers}\label{sec:dsp}
We focus on minimizing a loss function $F : \bbR^{d} \rightarrow \bbR$, that is, optimization problems of the following form
$$    \min_{x \in \mathbb{R}^d} F(x).
$$
Throughout the paper, we denote by $f : \bbR^{d} \rightarrow \bbR$ a random function computed on a minibatch sampled from the underlying data distribution. Therefore, for any $x\in \bbR^{d}$, we have  $\bbE \left[f(x)\right] = F(x)$. Moreover, we assume $F$ is differentiable and that $\bbE \left[\nabla f(x)\right] = \nabla F(x)$ for all $x\in \bbR^{d}$. We use $f_t$ to denote the random function evaluated with input mini-batch sampled at $t$. 


As discussed in the introduction, various optimizers were proposed in the literature, each excelling in different scenarios. The main goal in this work is to exploit the relationships between given optimizers towards the development of a meta-optimizer that automatically adapts the optimizer choice to the learning task. Informally, a parameterized optimizer can be described as the convex hull of a set of optimizers, when mapped to a Euclidean space under a certain parameterization. This section is devoted to explaining this notion in more detail, before Section~\ref{sec:mada} discusses how to perform learning in this Euclidean space.


In order to build a parameterized optimizer, we begin with a collection of existing optimizers. For the sake of concreteness, we illustrate the idea through the following four optimizers: Adam, AMSGrad, Adan and Yogi. 
A careful inspection reveals that these four optimizers can be described in one update rule that involves three iteratively generated sequences: model parameters, first-order moments, and second-order moments, which will be denoted throughout using the vectors $x$, $m$, and $v$, respectively. Each optimizer only differs in the way it updates these sequences. More precisely, starting with any $x_{0} \in \bbR^{d}$ and setting $v_{0} = m_{0} = 0$, the generic update rule is given by
\begin{align} \label{eq:generic_update}
    x_{t} = x_{t-1} - \alpha_{t}\frac{m_{t}}{\sqrt{v_{t}} + \epsilon},
\end{align}
where $\alpha_{t} > 0$ is the learning rate and $\epsilon > 0$ is given. \cref{tab:methods} shows how each optimizer defines its first- and second-moment iterates within this formulation, where we define $g_{t} := \nabla f_{t}(x_{t-1})$. For simplicity of the exposition, we omit the bias-correction terms for all optimizers.

\begin{table*}[t]
\centering
\begin{tabular}{cccc}
\hline
  Method & First-Order Moment & Second-Order Moment \\ \hline
  Adam \citep{KingmaB14} & $m_{t} = \beta_1 m_{t-1} + (1-\beta_1) g_t$  & $v_{t} = \beta_2 v_{t-1} + (1-\beta_2) g_t^2$ \\ \hline
  AMSGrad \citep{reddi2018} & $m_{t} = \beta_1 m_{t-1} + (1-\beta_1) g_t$ & \begin{tabular}{@{}c@{}} $\bar v_t = \beta_2 \bar v_{t-1} + (1-\beta_2) g_t^2$ \\ $v_{t} = \max\{v_{t-1}, \bar v_t\}$\end{tabular}   \\ \hline
  Adan \citep{xie2023adan}  & \begin{tabular}{@{}@{}c@{}} ${\bar m}_{t} = \beta_1 {\bar m}_{t-1} + (1-\beta_1) g_t$ \\ $n_t = \beta_3 n_{t-1} + (1-\beta_3)(g_t-g_{t-1})$ \\ $m_t = \bar{m}_t + \beta_3 n_t$ \end{tabular} & \begin{tabular}{@{}c@{}} $\hat g_t = g_t + \beta_3(g_t - g_{t-1})$ \\ $v_t = \beta_2 v_{t-1} + (1-\beta_2) \hat g_t^2$ \end{tabular} \\ \hline
  Yogi \citep{zaheer2018adaptive} & $m_{t} = \beta_1 m_{t-1} + (1-\beta_1) g_t$ & \begin{tabular}{@{}c@{}} $\hat g_t = v_{t-1}+ g_t^2\cdot\text{sign}( g_t^2 - v_{t-1})$ \\ $v_t = \beta_2 v_{t-1} + (1-\beta_2) \hat g_t$ \end{tabular} \\ 
\end{tabular} 
\captionof{table}{A unified framework to express adaptive moment optimizers.} \label{tab:methods}
\end{table*}
To be able to parameterize the design space of these four optimizers, we introduce real coefficients restricted between $0$ and $1$ that interpolate terms arising from different optimizers. Particular choices of these coefficients (typically at extreme values of 0 and 1) recover the underlying optimizers, but they express a new optimizer for their intermediate values. 

As an example, observe that the first-order moment update rule of Adan already subsumes those of Adam, AMSGrad, and Yogi (where $\{\bar{m}_t\}$ and $\{n_t\}$ are new sequences
introduced by Adan):
\begin{align}
    {\bar m}_{t} & = \beta_{1}{\bar m}_{t - 1} + (1 - \beta_{1})g_{t} \nonumber \\ 
    n_t & = \beta_{3}n_{t - 1} + (1 - \beta_{3})(g_{t} - g_{t - 1}) \nonumber \\ 
    m_t & = \bar{m}_{t} + \beta_{3}n_{t}, \label{fstAdan}
\end{align}
where taking $\beta_{3} = 0$ recovers the update rules of the other three optimizers. With respect to the second-order moments, Adan similarly covers Adam when $\beta_{3} = 0$, since we get that ${\hat g}_{t} = g_{t}$. 
In order to incorporate the second-order moments of AMSGrad and Yogi in the same parameterization, we introduce two new coefficients $c , \rho \in [0, 1]$, and unify the second-moment computation as follows:
\begin{align}
    {\hat g}_{t} & = g_{t} + \beta_3(g_{t} - g_{t-1}) \nonumber \\
    {\tilde g}_{t}^2 & = c{\hat g}_{t}^{2} + (1 - c)(v_{t-1} + {\hat g}_{t}^{2}\cdot\text{sign}({\hat g}_{t}^{2} - v_{t-1})) \nonumber \\
    {\tilde v}_{t} & = \beta_{2}{\tilde v}_{t - 1} + (1 - \beta_{2}){\tilde g}_{t}^{2} \nonumber \\
    v_{t}^{(max)} & = \max\{ {v_{t-1}^{(max)} , \tilde v}_{t} \} \nonumber \\
    v_{t} & = \rho{\tilde v}_{t} + (1 - \rho)v_{t}^{(max)}. \label{2ndPara}
\end{align}
It is easy to check that, for instance, the second-order moments of Adam can be recovered when $\beta_{3} = 0$, and $c = \rho = 1$; those of AMSGrad are recovered when $\beta_{3} = \rho = 0$ and $c = 1$; Adan can be recovered when $c=\rho=1$; and Yogi can be recovered with $\beta_3=c=0$ and $\rho=1$.

To summarize, in this example, for the four optimizers Adam, AMSGrad, Adan and Yogi, our parameterized optimizer is given by the three updating rules: \eqref{eq:generic_update}, \eqref{fstAdan} and \eqref{2ndPara}. Following the same line of arguments one can generate parameterized optimizers for other collections of given optimizers as well. However, unless we know a priori how to set the corresponding coefficients, the parameterized optimizer is not readily usable in practice. In the next section, we develop our meta-optimizer which uses a parameterized optimizer and learns the coefficients in an online fashion.

\section{Meta-Adaptive Optimizers} \label{sec:mada}
In the previous section, we described how one can parameterize the design space of a given collection of optimizers. Here, we define \Mada as a meta-optimization framework that dynamically learns the coefficients of a parameterized optimizer during training. In this work, we use hyper-gradient descent \citep{baydin2018online,chandra2022gradient} to learn these coefficients, which avoids an expensive hyper-parameter optimization loop that involves multiple training runs. However, we note that in principle other techniques can also be combined with parameterized optimizers to learn the coefficients.

\textbf{Learning interpolation coefficients.} Hyper-gradient descent views the hyper-parameters as trainable parameters of the loss function, and thus differentiates the loss with respect to them and updates them through gradient steps. \cite{chandra2022gradient} re-purposes PyTorch \texttt{autograd} machinery \citep{paszke2017automatic} to automatically compute gradients with respect to optimization hyper-parameters such as learning rate and Adam $\beta_1$ and $\beta_2$ parameters. Since the update rule of a parameterized optimizer is differentiable with respect to its learnable coefficients, we can apply the same approach to update them using hyper-gradient descent steps. 

Now, we describe our meta-optimizer \Mada in detail. 
We will denote a parameterized optimizer by $\mathcal{O}_{q}$, where $q \in \mathcal{D}$ denotes the vector of coefficients that defines the optimizer, and $\mathcal{D}$ represents the domain of the vector $q$. In the case of the example from Section~\ref{sec:dsp}, we have that $q = (\beta_1, \beta_2, \beta_3, \rho, c)$ and $\mathcal{O}_{q}$ is given by \eqref{eq:generic_update}, \eqref{fstAdan}, and \eqref{2ndPara}. Note that $\mathcal{O}_q$ also encapsulates non-learnable state parameters (such as first-order and second-order moments) and other hyper-parameters such as the learning rate, weight decay, and stability parameter. The domain $\mathcal{D}$ represents the set of values that the vector $q$ is allowed to take in the parameterization. In the example from Section~\ref{sec:dsp}, the domain is the unit hypercube where each element is in the range $[0, 1]$. We denote by $\Pi_{\mathcal{D}}$ the orthogonal projection onto the set $\mathcal{D}$. We provide a pseudo code to illustrate this (see Algorithm \ref{alg:mada}).
\begin{algorithm}[h]
\caption{Pseudocode for a generic \Mada}
    \label{alg:mada}
\textbf{Input:} A parameterized optimizer $\mathcal{O}_{q}$, where $q \in \mathcal{D}$, a hyper-learning rate $\alpha$, number of total iterations $T$. \\
\textbf{Init.:} $x_0$ and $q_{0}$.
\begin{algorithmic}[1]
    \FOR{t=1 to T}
        \STATE {Sample $f_t$.}
        \STATE {Update the model parameters:} 
        $$x_t = \mathcal{O}_{q_{t-1}}(x_{t-1}).$$
        \STATE {Update the optimizer coefficients:}
        $$q_t = \Pi_{\mathcal{D}}\left[q_{t-1} - \alpha \nabla_{q} f_t(x_{t-1})\right].$$
    \ENDFOR
\end{algorithmic}
\textbf{Output:} Model wights $x_{T}$. 
\end{algorithm}

\textbf{Hyper-gradient computation.} Before concluding this section, we briefly illustrate hyper-gradient computation\footnote{Further details on hyper-gradients can be found in \citep{baydin2018online,chandra2022gradient}.}. In order to compute the hyper-gradient of the loss with respect to a particular optimizer coefficient, we treat the updated model weights as a function of the coefficient. For instance, considering again the example from Section~\ref{sec:dsp}, we will show how to compute the gradient of the function $f_{t}$ with respect to the coefficient $\rho$ using the parameterized optimizer as given in \eqref{2ndPara}. Using the chain rule, we obtain
\begin{align}
    \frac{\partial f_{t+1}(x_t)}{\partial \rho} & = \frac{\partial f_{t+1}(x_t)}{\partial x_{t}}\cdot\frac{\partial x_{t}}{\partial v_{t}}\cdot\frac{\partial v_{t}}{ \partial \rho} \nonumber \\
    &= \frac{\partial f_{t+1}(x_t)}{\partial x_{t}}\cdot\frac{\partial (x_{t-1} - \alpha_{t}\frac{m_t}{\sqrt{v_t}+\epsilon})}{\partial v_{t}}\cdot\frac{\partial v_{t}}{ \partial \rho} \nonumber \\
    &=\frac{\partial f_{t+1}(x_t)}{\partial x_{t}}\cdot\frac{\alpha_t m_{t}}{2\sqrt{v_t}(\sqrt{ v_{t}}+\epsilon)^2}\cdot\left(\tilde v_t-v_{t}^{(max)}\right) \label{eq:hypergrad},
\end{align} 
where the first term $\partial f_{t+1}(x_t)/\partial x_{t}$ is readily computed using standard back-propagation\footnote{In \eqref{eq:hypergrad}, the first term is a row vector, the second term is a Jacobian matrix, and the third term is a column vector, and thus the $\cdot$ operator refers to a matrix multiplication.}. Note that the latter two objects of \eqref{eq:hypergrad} are not explicitly constructed in practice; instead \texttt{autograd} simply \emph{continues} back-propagating the already-computed parameter gradients into optimizer coefficients, while handling the necessary book-keeping for hyper-gradient computation. 

\section{On Convergence of Interpolated Optimizers}\label{sec:avgrad}










Since \Mada uses hyper-gradients to update the optimizer, a prerequisite for a base optimizer to work well within \Mada is that its update rules must allow efficient flow of hyper-gradients to the parameterization coefficients. In particular, in our experiments with \Mada we have found that including AMSGrad among the base optimizers has an adverse effect on the end performance of the trained model. We conjecture that this is caused by the maximum operator in the second-moment term. Specifically, note that back-propagation of gradients through $\max(a, b)$ corresponds to a simple routing operation to the larger one of $a$ and $b$, where the smaller one is passed $0$ gradients. In AMSGrad, $v_t = \max\left\{ v_{t-1}, \bar v_t\right\}$ which means that for most steps the first term will be greater, causing insufficient hyper-gradient updates on $\beta_2$ parameter through $\bar v_t$ path\footnote{To be accurate, considering the example parameterization \eqref{2ndPara}, $\tilde v_t$ term provides another path for hyper-gradients on $\beta_2$. However, in practice we observed that when AMSGrad is used, $\rho$ parameter tends to vanish, diminishing the effect of this path as well.}. To remedy this, we introduce AVGrad (formally defined in next section), which replaces the maximum operator with time-averaging of second moments, and results in better hyper-gradient flow and validations loss.

\subsection{AVGrad and its interpolation with Adam} 


Following the  notation in \cref{tab:methods}, we define AVGrad as an optimizer where the first-order moments and second-order moments are defined by
\begin{align}
    m_t &=\beta_1 m_{t-1}+(1-\beta_1) g_t \notag \\
    \bar v_t &= \beta_2 \bar v_{t-1}+(1-\beta_2)g_t^2 \notag \\
    \tilde v_t &= \frac{1}{t}(\bar v_{t}+(t-1) \tilde v_{t-1}), \notag \\
    v_t &= \tilde v_t \label{eq:lastline},
\end{align}
where $\tilde v_t$ is the running average of past $\overline v_t$'s. We provide the convergence analysis of AVGrad in Appendix~\ref{app:avgrad} in \cref{thm1,thm2}  

In what follows, we focus on the interpolation between Adam and AVGrad, with the interpolation coefficient $\rho_t$. Specifically, we replace last line in \eqref{eq:lastline} with $v_t := \rho_t \bar v_t + (1-\rho_t) \tilde v_t$ where $\rho_t$ interpolates between second-order moments of Adam and AVGrad. 

\textbf{Soundness of AVGrad.} 
We start with a proposition showing that AVGrad is a valid alternative to AMSGrad in mitigating the sub-optimal convergence issue in Adam \citep{reddi2018}. Specifically, when $\rho_t$ decays with $1/t$ (\emph{i.e.}, converges to AVGrad), the interpolated optimizer fixes the non-decreasing learning rate problem in Adam, similar to AMSGrad.
\begin{proposition} \label{prop:avgrad}
    The following inequality is a sufficient condition for non-increasing effective learning rate:
    $$
        \rho_t \leq \frac{1}{t(1-\beta_2)+1}.
    $$
\end{proposition}
In Appendix~\ref{app:experiments}, we use \Mada to solve the convex problem given in \citep{reddi2018}, an example where Adam fails. We found that \Mada quickly converges to AVGrad, and thus to the optimum, by adapting $\rho_t \rightarrow 0$, even when we initialize it from Adam ($\rho_0 = 1$).


\subsection{Convergence of the interpolated optimizer}


Due to the complexity of our overall parameterization, deriving theoretical guarantees for \Mada at its full scope is challenging. Therefore, in this section, we reduce the scope by examining convergence properties of interpolations between two optimizers, namely, AVGrad and Adam. In this setting, we will show that the parameterized optimizer allows the design of novel interpolated optimizers which improve the convergence rate of either optimizer up to constant factors, demonstrating the value of the parameterized optimizer formulation. We defer the proof and other technical details to Appendix~\ref{app:thm interp}.

We make the following standard assumptions: $F$ is bounded from below, is $L$-smooth, and with stochastic gradients that are bounded almost surely.

\cite{defossez2022a} proposed a unified analysis for the convergence analysis of Adagrad and Adam in the non-convex setting. Since AVGrad can be seen as a version of Adagrad, the proof technique of \cite{defossez2022a} can be adapted to AVGrad with some modifications. The modification includes the introduction of $\rho$ parameter which interpolates between the second-order moment and the moment average, and results in an additional degree of freedom. Subsequently, $\rho$ can be chosen to skew the updates towards the advantageous term under different scenarios.

Here we state the convergence result of the interpolated optimizer without momentum, that is when $\beta_1 = 0$ (see precise statement in Appendix~\ref{app:thm interp}). The extension to the case with momentum can be done in the same way we extend Theorem~\ref{thm1} to Theorem~\ref{thm2}, similar to \citep{defossez2022a}. 

\begin{theorem}[Convergence of interpolation of AVGrad and Adam without momentum] \label{thm3} Under the assumptions above and $\alpha_t=\frac{\alpha}{\sqrt{t}}$ for some $\alpha > 0$, and for $\rho_t = \frac{\rho}{t}$ for a constant $\rho$:
$$
G_T \leq E(T) \left[ \frac{C_1}{T} + \frac{C_2 \ln \left( E(T)\right)}{T} + C_3 \Big[\ln\Big(\frac{\rho}{\beta_2}\Big)\Big]_+ \right],
$$
where, $G_T = \frac{1}{T}\sum_{t=1}^T \|\nabla F(x_t) \|^2$, $E(T) = \frac{\sqrt{\rho+(1-\rho)T}}{\sqrt{1-\beta_2}}$, and $C_1,C_2,C_3$ are constants independent of $T,\rho,\beta_2 $.      


\end{theorem}

\begin{figure*}[h]
\begin{minipage}{.45\linewidth}
\centering
    \includegraphics[scale=0.5]{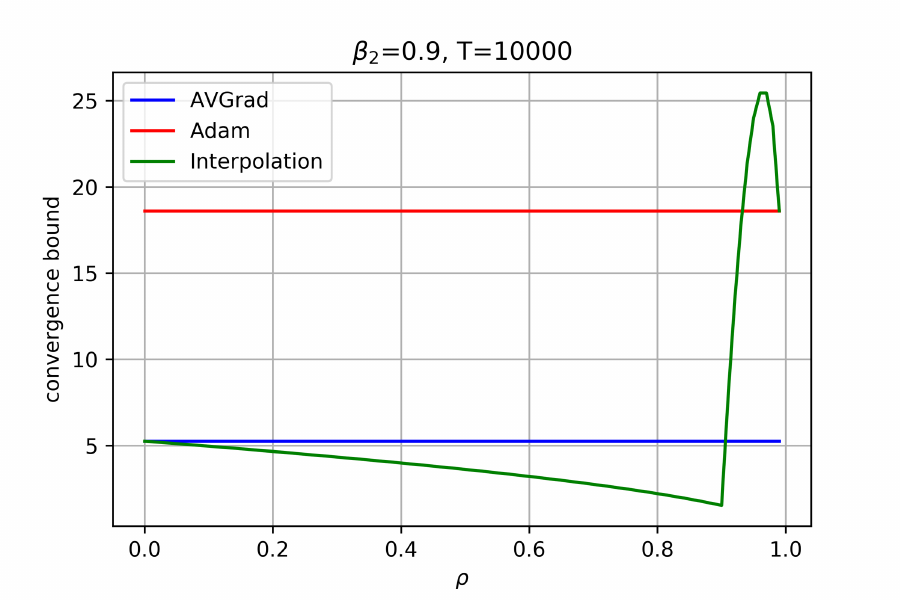}
    \caption{Value of the bound in Theorem~\ref{thm3} for a representative case where $\beta_2=0.9, T= 10,000$.}
    \label{fig:bound1}
\end{minipage}\hfill
\begin{minipage}{.45\linewidth}
\centering
  \includegraphics[scale=0.5]{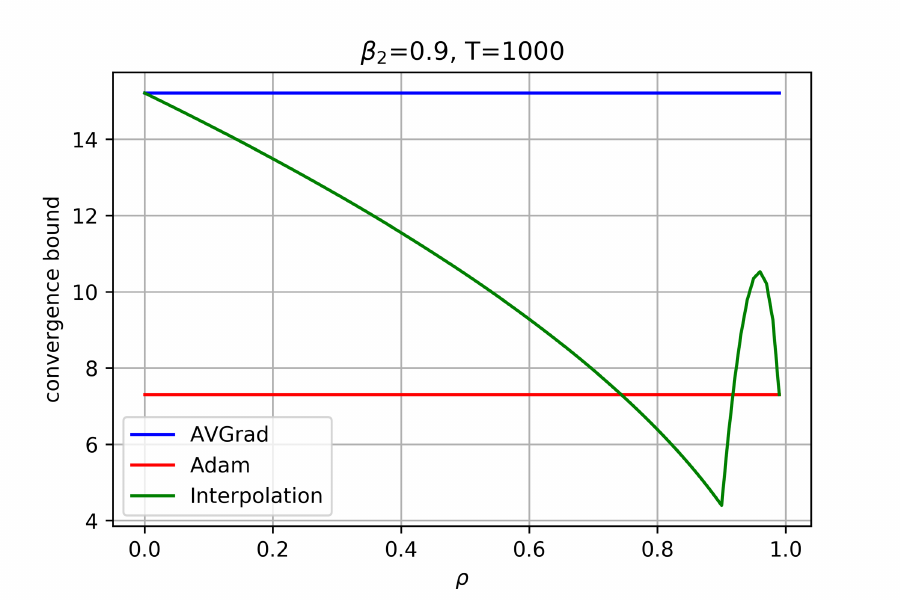}
    \caption{Value of the bound in Theorem~\ref{thm3} for a representative case where $\beta_2=0.9, T= 1,000$.}
    \label{fig:bound2}
\end{minipage}\hfill
\end{figure*}

\textbf{Remark.} Aligned with the condition in Proposition~\ref{prop:avgrad}, the contribution from $v_t$ in Theorem~\ref{thm3} is multiplied with $1/t$ in order to avoid divergent terms in the bound. Observe that when $\rho=0$ or $\rho=1$ we recover the convergence rates for AVGrad (see Theorem~\ref{thm1} in Appendix~\ref{app:avgrad}) and Adam \citep{defossez2022a} respectively. 
We note that for the right-hand side to remain bounded, one would need either $0 \leq \rho \leq \beta_2$ (all terms vanish as in AVGrad), or $\rho=1$ (a constant term remains as in Adam).
To gain more insights, in \cref{fig:bound1,fig:bound2}, we plot the value of the bound in Theorem~\ref{thm3} with respect to $\rho$ for two representative cases . The first case with larger $T$ represents a more favorable setting for AVGrad as all the terms vanish with $T$. The second is more favorable for Adam since the vanishing terms vanish faster than AVGrad. In both cases, the interpolated optimizer provides the best convergence bound when $\rho = \beta_2$, which suggests an advantage for the interpolated optimizers.

\section{Experiments}\label{sec:experiments}
In our experiments, we aim to answer how the generalization and downstream performance of \Mada compares against fixed optimizers, and how robust \Mada is to poor hyper-parameter initializations compared to fixed optimizers. We focus on the pre-training of models from scratch. We also try to gain insights into the behavior of \Mada by monitoring the evolution of the optimizer coefficients. Additional details of our experimental setup, and results can be found in Appendix~\ref{app:experiments}.
\begin{figure*}[htbp]
\begin{minipage}{.45\linewidth}
\centering
    \includegraphics[scale=0.5]{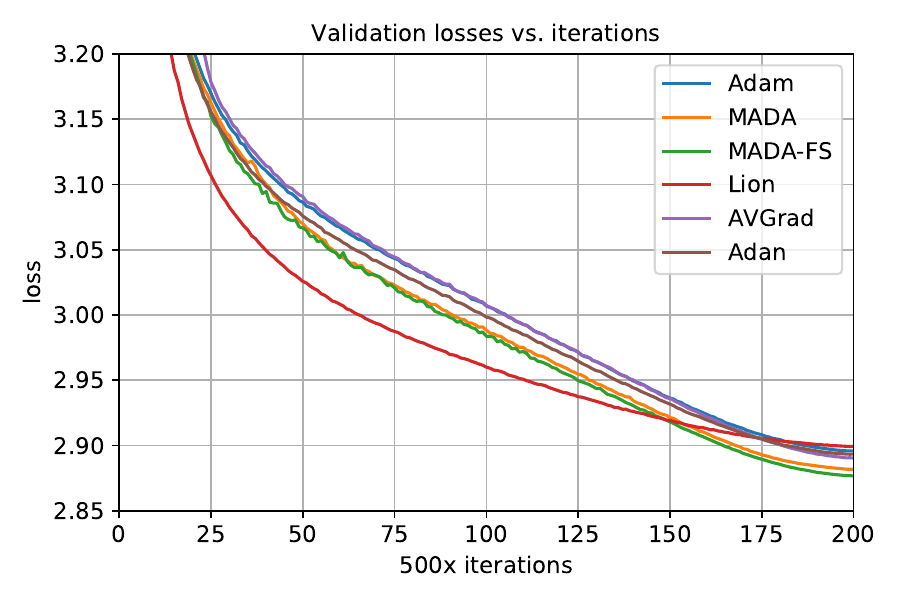}
    \caption{Validation losses of competing methods on OpenWebText for GPT-2 (125M) model using the same random seed.}
    \vspace{-8pt}
    \label{fig:gpt_val}
\end{minipage}\hfill
\begin{minipage}{.45\linewidth}
\centering
    \includegraphics[scale=0.5]{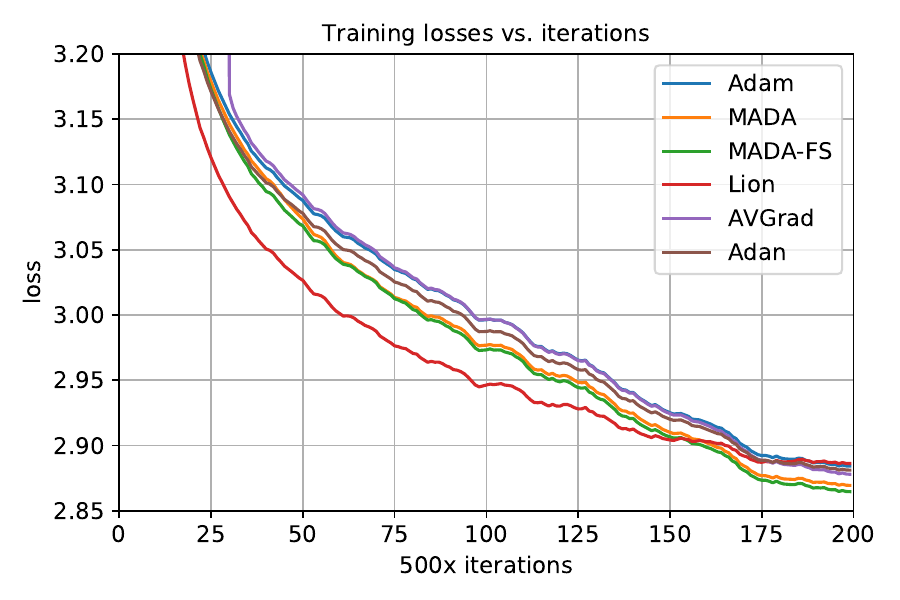}
    \caption{\kaancr{Training losses of competing methods on OpenWebText for GPT-2 (125M) model using same random seed, smoothed for convenience. }}
    \vspace{-8pt}
    \label{fig:gpt_train}
\end{minipage}\hfill
\end{figure*}

\begin{figure*}
\begin{minipage}[t]{.45\linewidth}
\captionof{table}{Validation loss of \Mada on OpenWebText vs other adaptive optimizer baselines.}
    \label{tab:gptcomp}
\centering
\begin{tabular}[t]{lc} \hline
  Method & Validation Loss \\ \hline
  Adam &  2.8956 \\
  Adan & 2.8896\\
  HyperAdam & 2.8950 \\
  Lion & 2.8892 \\
  AVGrad & 2.8895 \\ 
  \midrule
  \Mada & \textbf{2.8806} \\
  \Mada-FS & \textbf{2.8766} \\ \hline
  \multicolumn{2}{@{}l}{\textbf{Poor initialization}}\\
  $\Mada^-$ & 2.8921 \\
  $\text{Adam}^-$ & 2.9157 \\ \hline
  \end{tabular}
\end{minipage}\hfill
\begin{minipage}[t]{.5\linewidth}
\captionof{table}{Validation perplexities of competing methods on OpenWebText, Wikitext and Lambada datasets.}
    \label{tab:gptppl} 
\centering 
\begin{tabular}{lccc} \cmidrule(lr){1-4} 
  Method & OpenWebText & Wikitext & Lambada \\ \cmidrule{1-4}
  Adam &  18.0940 & 63.8544 & 77.3314 \\
  Adan & 17.9863 & 63.5518 & 74.6970\\
  HyperAdam & 18.0843 & 61.7717 & \textbf{72.6803} \\
  Lion & 17.9792 & 61.8661 & 75.3158\\
  AVGrad & 17.9840 & 64.2620 & 75.1317\\
  \cmidrule{1-4}
  \Mada  & \textbf{17.8249}& 61.2513 & 74.2480 \\
  \Mada-FS & \textbf{17.7544}& 59.4086 & \textbf{73.5623} \\
  \cmidrule{1-4}
  \multicolumn{2}{@{}l}{\textbf{Poor initialization}}\\
  $\Mada^-$ & 18.0317 & \textbf{57.1613} & 75.3550 \\
  $\text{Adam}^-$ & 18.4624 & 72.9017 & 79.1217 \\ 
  \end{tabular}
\end{minipage}
\end{figure*}

\vspace{-8pt}
\subsection{A concrete parameterized optimizer}
We use a concrete parameterization that is similar to the example presented in Section~\ref{sec:dsp}, with two changes. First, we replace AMSGrad with AVGrad, and second, we include Lion among the optimizers that we interpolate. Hence, second-order moment in \eqref{2ndPara} becomes $v_t = \rho \bar v_t + (1-\rho) \tilde v_t$ as in the interpolation in Section~\ref{sec:avgrad}. Moreover, the update term becomes $\gamma \frac{m_{t}}{\sqrt{v_{t}} + \epsilon} + (1-\gamma)\text{sign}(u_t)$ where $u_t$ is the moment term from Lion \citep{chen2023symbolic}. We refer the reader to Appendix~\ref{app:experiments} for a complete description of the parameterization. Lion was omitted in the example in Section 2 for the sake of simplicity, since it integrates less naturally with the rest of the optimizers.

\subsection{Experimental setting} 
\textbf{Data and models.}
In recent years, auto-regressive models have been widely used for benchmarking and algorithm evaluation \citep{radford2019language, liu2023sophia}. Motivated by this, we evaluate \Mada on the causal language modeling task with GPT-2, over two datasets: Shakespeare \citep{karpathy2015}, and OpenWebText \citep{Gokaslan2019OpenWeb}. 
On the Shakespeare dataset, we train a 6 layer transformer with context size of 256 (11M parameters) and \kaancr{fine-tune GPT-2 (XL) with 48 layers (1.5B parameters)}. On OpenWebText, we train GPT-2 (small) with 12 layers, 1024 context size (125M parameters) and GPT-2 (medium) with 24 layers, 1024 context size (335M parameters). 
\footnote{We use nanoGPT (\url{https://github.com/karpathy/nanoGPT}) code base for the implementation.}

\textbf{Baselines and learned optimizers.} On OpenWebText, we compare \Mada to Adam \citep{KingmaB14}, HyperAdam \citep{chandra2022gradient} (Adam with $\beta_1$ and $\beta_2$ parameters learned through hyper-gradients), Adan \citep{xie2023adan}, Lion \citep{chen2023symbolic}, and AVGrad. \kaancr{For all the methods we used decoupled weight decay \citep{loshchilov2018decoupled}, i.e. the AdamW variant of Adam is used.}
We keep the learning hyper-parameters such as learning rate schedule and weight decay identical, except for Lion where we choose a lower initial learning rate as suggested in \citep{liu2023sophia}. For a fair comparison, for all optimizers, we use the same codebase based on \cite{chandra2022gradient}. For OpenWebText experiments, we use established parameters for Adam ($\beta_1=0.9,\beta_2 = 0.95, \epsilon=10^{-6}$) and also use these values as the initial parameters for \Mada, HyperAdam, and AVGrad; for other methods we use the parameters suggested in respective papers; and measure the validation loss. For Shakespeare experiments, we compare \Mada to Adam and HyperAdam; the relatively small model size in this experiment allows us to sweep the grid of initial $\beta$ parameters and compare the final training loss across many different hyper-parameter choices. In some experiments, we also evaluate the final state of \Mada statically, \emph{i.e.}, we take the final optimizer learned by \Mada and use it as the fixed optimizer from the beginning of training, which we refer to as \Mada-FS.

\begin{figure*}[htbp]
\begin{minipage}{.3\linewidth}
\centering
    \includegraphics[scale=0.33]{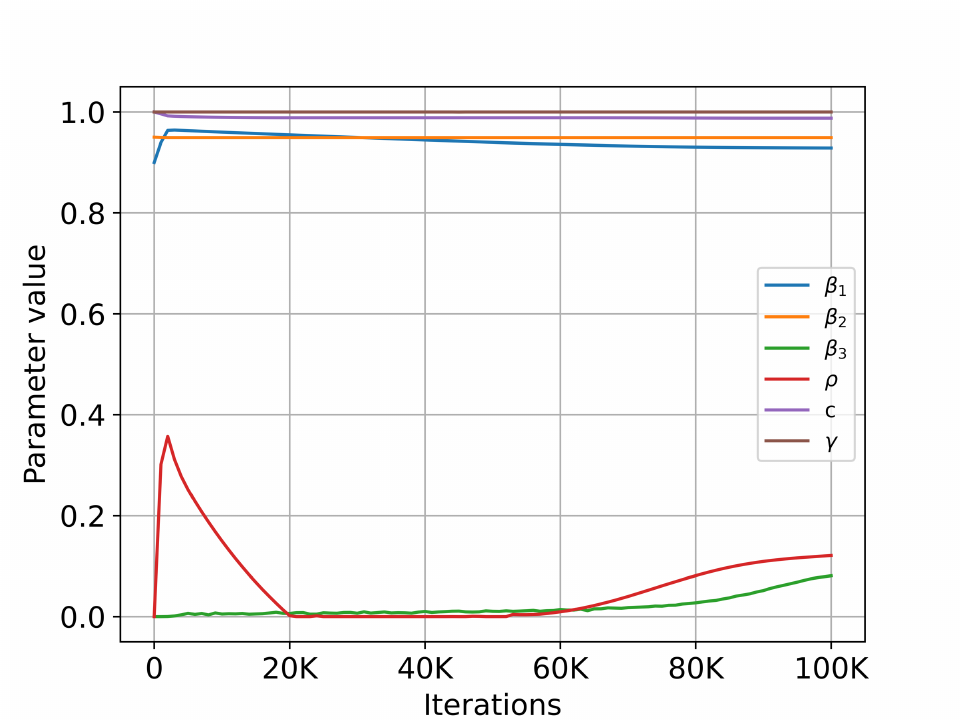}
    \caption{Parameter evolution for $\beta_{1,0}=0.9,\beta_{2,0}=0.95,\beta_{3,0}=0$.}
    \label{fig:gpt_parameters}
\end{minipage}\hfill
\begin{minipage}{.3\linewidth}
\centering
  \includegraphics[scale=0.33]{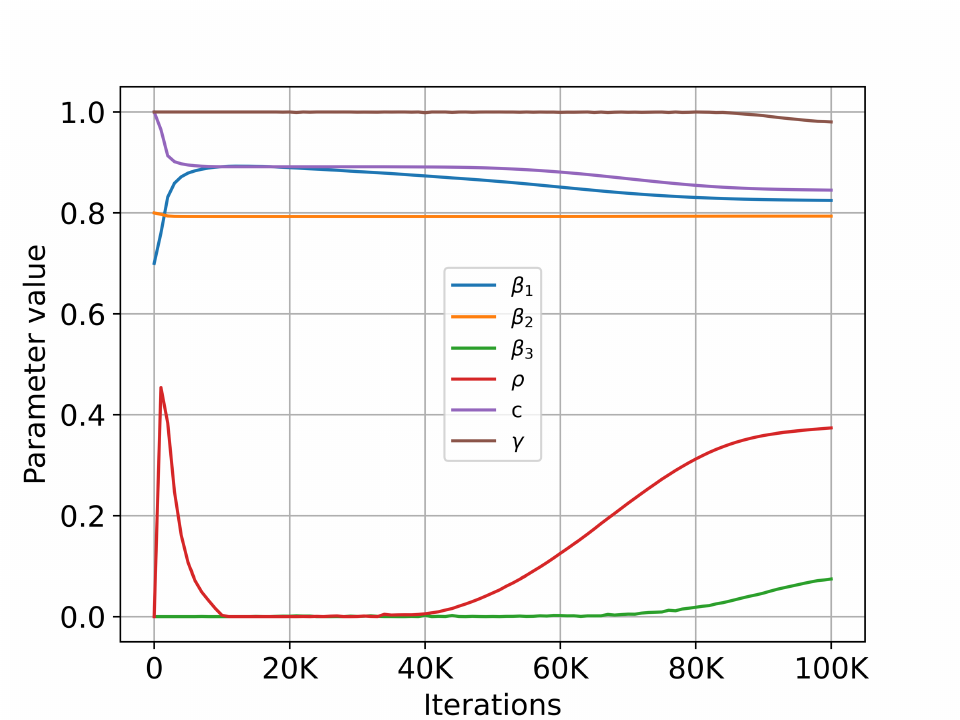}
    \caption{Parameter evolution for $\beta_{1,0}=0.7,\beta_{2,0}=0.8,\beta_{3,0}=0$.}
    \label{fig:gpt_parameters2}
\end{minipage}\hfill
\begin{minipage}{.3\linewidth}
\centering
  \includegraphics[scale=0.33]{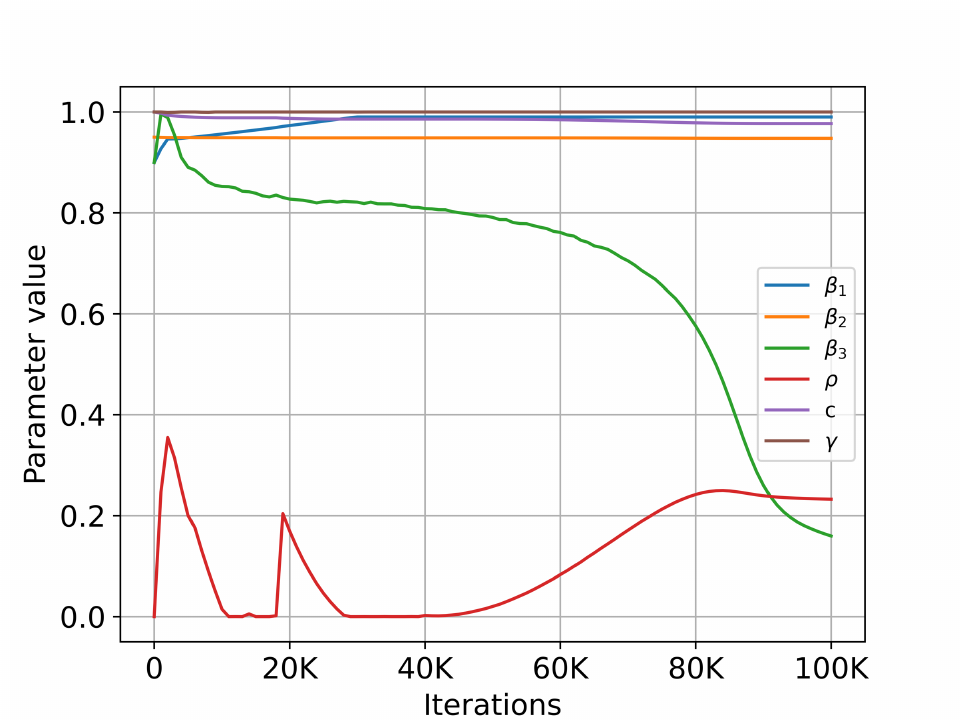}
    \caption{Parameter evolution for $\beta_{1,0}=0.9,\beta_{2,0}=0.95, \beta_{3,0}=0.9$.}
    \label{fig:gpt_parameters3}
\end{minipage}\hfill
\end{figure*}

\kaancr{\textbf{Tuning hyper-learning rates.} While hyper-learning rates are additional hyper-parameters to be tuned; the tasks are less sensitive to hyper-learning rates compared to other hyper-parameters such as learning rate, most likely because each hyper-learning rate controls the update of a single parameter. As a result, hyper-learning rates can be fine-tuned easily and can be transferred across similar tasks. In general, for the hyper-learning rates of $\beta_1,\beta_2$, since the gradients are relatively large, we search in a set of smaller values: $\{10^{-3},5\times10^{-4}, 10^{-4}\}$. For the other parameters we search in the set $\{10^{-1},5\times10^{-2}, 10^{-2}\}$.}
\subsection{Results}
In this section, we first present empirical results of training GPT-2 models on OpenWebText, and provide a discussion of generalization performance and robustness of \Mada. We also show how the interpolation coefficients evolve during training. Next, we provide a visualization of the optimal training loss given various initializations of $\beta_1,\beta_2$ values using the smaller Shakespeare dataset. We present additional results in Appendix~\ref{app:experiments} including a toy problem to show that \Mada discovers optimal solutions that Adam cannot find, and comparing \Mada and Adam for GPT-2 (medium).

\textbf{GPT-2 on OpenWebText.} 
When we initialize \Mada from AVGrad and Adan with $\beta_{1,0}=0.9,\beta_{2,0} = 0.95,\beta_{3,0}=0.9, \rho_{0}=0$ (where subscript $0$ denotes the time index) from Table~\ref{tab:gptcomp} we observe that \Mada consistently outperforms Adam and other recently proposed optimizers, including Lion, Adan, and HyperAdam. In particular, \Mada-FS outperforms Adam with established parameters by 0.019 in validation loss which is a significant improvement for this task. \kaancr{In \cref{fig:gpt_val} and \cref{fig:gpt_train} we see that \Mada is able to converge to a lower loss than baseline methods (both in terms of validation and training loss). Note that MADA is able to converge faster than AVGrad and Adam with the same learning rate and schedule, which may be attributed to theoretical result \cref{thm3} on interpolated optimizers. } 

\textbf{Perplexity on benchmark datasets.} To measure the generalization of \Mada to other datasets, we compute the validation perplexity of the trained models on three datasets as shown in Table~\ref{tab:gptppl}. We see that on OpenWebText \Mada outperforms Adam and HyperAdam by 0.34 and 0.33 points); while on Wikitext \citep{wikitext}, \Mada outperforms these two baselines by 4.45 and 2.37 respectively. On Lambada \citep{lambada} it is the second-best-performing method with a small gap behind HyperAdam (0.88). 

\textbf{\Mada-FS versus \Mada.} 
We find that \Mada-FS performs slightly better than \Mada for OpenWebText (0.004 in validation loss and 0.07, 1.85, 0.68 points on validation perplexity on OpenWebText, Wikitext and Lambada). More details are in \cref{tab:gptcomp,tab:gptppl}. On the other hand, on Shakespeare we observe that, the dynamically-evolving optimizer performs generally better than using the fixed final optimizer. We conjecture that the difference in the number of iterations, the model size or dataset may explain this phenomenon. 

\textbf{Learning interpolation coefficients is crucial.} We also see that HyperAdam performance is very close to Adam (2.8950 vs. 2.8956), which suggests that the main performance improvement of \Mada does not simply originate from the tuning of the $\beta_1, \beta_2$ parameters, but rather from the adaptation of the interpolation coefficients in the optimizer space.
\begin{figure*}[htbp]
    \centering
    \includegraphics[scale=0.5]{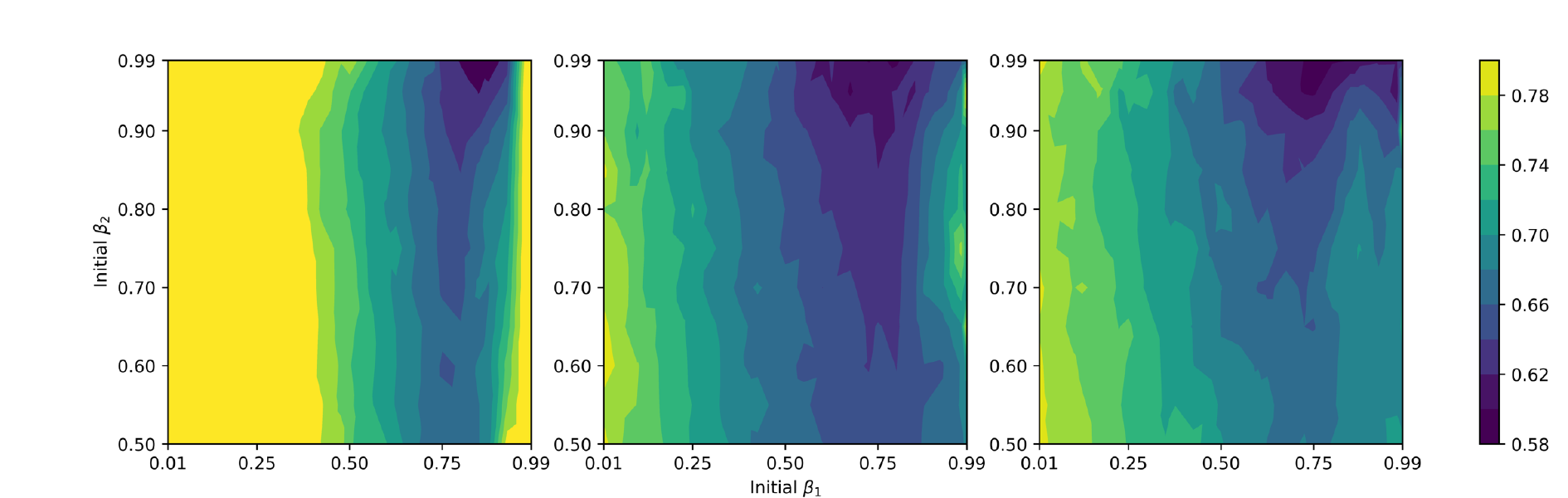}
    \vspace{-8pt}
    \caption{Final training loss of Adam, \Mada, and HyperAdam vs. initial $\beta$ values on Shakespeare dataset. \Mada yields better performance for wider choice of initial $\beta$ values,
illustrating its robustness.}
    \label{fig:shakespeare}
\end{figure*}

\textbf{Robustness to initial hyper-parameters.} Additionally, we compare \Mada to Adam when we initialize from the suboptimal hyperparameters $\beta_{1,0}=0.7, \beta_{2,0}=0.8$ (we call these optimizers $\Mada^-$ and $\text{Adam}^-$ respectively)
. Notably, we observe that even suboptimally initialized $\Mada^-$ is able to outperform Adam with the established initial parameters, as well as HyperAdam. This particular example indicates the robustness of \Mada. We provide more examples with poor initial hyper-parameter choices on Shakespeare dataset in the next subsections.
\vspace{-4pt}

\textbf{Evolving hyper-parameters.} We monitor the evolution of the interpolation coefficients during training of GPT-2 (125M) model on OpenWebText for various choices of initialization of $\beta_1, \beta_2$ (Figures ~\ref{fig:gpt_parameters}, \ref{fig:gpt_parameters2}, and \ref{fig:gpt_parameters3}). 

First, we observe that sub-optimal initialization of hyper-parameters (as in Figure~\ref{fig:gpt_parameters2}) results in a more significant update of the coefficients. 
This suggests that default optimizer state with $\beta_{1,0}=0.9,\beta_{2,0}=0.95$ is close to a local minimum. A closer inspection yields insight into how \Mada accelerates training. First, in Figures~\ref{fig:gpt_parameters},~\ref{fig:gpt_parameters2} we see that $\beta_1$ increases at the beginning which facilitates taking larger steps, while towards the end it decreases to give more emphasis on individual gradients. Second, we observe that the interpolation factor of AVGrad, $\rho$, follows a three-phase pattern. This suggests that \Mada puts more importance on individual normalization term in the beginning and towards the end, and it puts more emphasis on the less noisy averaging term during the intermediate stages.

Another interesting behavior we observe is when we initialize $\beta_3$ (which governs the weight of Adan) from a higher value such as $0.9$ (Figure~\ref{fig:gpt_parameters3}), we see $\beta_1$ reaches 0.99 and stays constant, in contrast with Figures~\ref{fig:gpt_parameters},~\ref{fig:gpt_parameters2}. Remarkably, \Mada initialized with $\beta_3=0.9$ automatically trains $\beta_1$ to be 0.99; i.e. \Mada, when initialized from combination of Adan and AVGrad, automatically finds the $\beta_1$ that is suggested by the authors of \citep{xie2023adan}. Note that $\beta_1, \beta_2, \beta_3$ in this work correspond to $1-\beta_1, 1-\beta_3, 1-\beta_2$ in \cite{xie2023adan}.
 
\textbf{Shakespeare dataset.} On Shakespeare dataset we compare \Mada to Adam and HyperAdam. We show that it results in significant increase in the performance across many initial $\beta_1,\beta_2$ values when training 11M parameter model. Particularly, in Figure~\ref{fig:shakespeare}, we see that \Mada results in better performance for the vast majority of initial $\beta_1$ and $\beta_2$ parameters, illustrating robustness. 

\textbf{Fine-tuning GPT-2 XL on Shakespeare dataset.} \kaancr{We also fine-tune the pre-trained GPT-2 XL (1.5B parameters) model on Shakespeare, for approximately 2 epochs, using \Mada, Adam, and HyperAdam methods. Table~\ref{tab:finetune} shows the resulting training loss values, where \Mada results in a large improvement. Notably, HyperAdam achieves no gain over Adam.} 

\begin{figure*}[h]
    \begin{minipage}[t]{0.45\linewidth}
    \centering
    \captionof{table}
      {
        Training losses after 2 epochs of fine-tuning on Shakespeare dataset.%
        \label{tab:finetune}
      }
    \begin{tabular}{lcc} \hline 
    Method & Training loss  \\ \hline
    \Mada &  \textbf{0.255}  \\
    Adam & 0.276   \\
    HyperAdam & 0.278
  \end{tabular}
\end{minipage}\quad
\begin{minipage}[t]{0.45\linewidth}
\centering
\captionof{table}
      {
        Test accuracy of competing methods for CNN and ResNet models.%
        \label{tab:vision}
      }
    \begin{tabular}{lcc} \hline 
    Method & 5-layer CNN & ResNet-9 \\ \hline
    \Mada &  \textbf{66.12} $\pm$ \textbf{0.14} & \textbf{93.79} $\pm$ \textbf{0.11} \\
    Adam & $65.84 \pm 0.11$ &  $93.73 \pm 0.10$  \\
    HyperAdam & $65.80 \pm 0.21$ &   $93.69 \pm 0.06$ \\
    Momentum SGD & $65.97 \pm 0.94$ & $92.60 \pm 0.10$
  \end{tabular}
\end{minipage}
\end{figure*}


\textbf{Vision tasks}. We also validated the performance of \Mada on two computer vision models: 5-layer CNN and ResNet-9 on CIFAR-10 dataset. We observe that \Mada shows consistent improvement over our baselines (see test accuracy results below in Table~\ref{tab:vision}). 

\section{Conclusion}\label{sec:discussions}
In this paper we propose an approach to unify a set of optimizers into a parameterized formulation. We further show we can use \Mada to dynamically learn the most suitable optimizer for a particular task during training using hyper-gradient descent. We employ \Mada to train language models and observe that it consistently outperforms Adam and other fixed optimizers both in terms of best validation performance, and in terms of robustness to initial hyper-parameters.

\textbf{Limitations.} The main limitation of our framework is the additional computational requirement and memory usage of the parameterized optimizer due to the maintained optimizer states. The additional computation can be broken down into two components (per-parameter computational steps and computation of hyper-gradients); both can be shown to be negligible with respect to the computational requirements of overall training (see Appendix~\ref{app:computation}, where we also provide the resulting memory use and iteration speed numbers). The additional memory usage arises from the fact that we need to maintain a larger optimizer state per parameter. To mitigate, one can employ sharded data parallelism techniques \citep{rajbhandari2020zero}, which shards the optimizer state across a large number of data-parallel devices. Another approach is to analyze the contributions of the constituent optimizers to the learning performance, and prune the ones that have little contribution.

\textbf{Future work.} Theoretically, we have shown the convergence of interpolation of AVGrad and Adam; as a future work we would like to construct a generic theoretical framework that can characterize the convergence of optimizer interpolations more generally. Empirically, we aim to gain a deeper understanding of the dynamics of the optimizer learning process in practice.

\bibliographystyle{plainnat}
\bibliography{bibliography}


\newpage
\appendix
\begin{center}
{\LARGE \bf Appendix}
\end{center}

\section{Setup, Details, and Proof for Theorem~\ref{thm3} } \label{app:thm interp}
\textbf{Assumptions.} (i) $F$ is bounded below by $F^*$, (ii)  stochastic gradients, that is $\forall x \in \bbR^d, \ \|\nabla f(x)\|_{\infty} \leq R \text{ a.s. }$, (iii) the objective function is smooth, that is $\forall x,y \in \bbR^d, \quad \|\nabla F(x) - \nabla F(y) \|_2 \leq L\| x-y\|_2$.

\textbf{Setup.} Our setup follows \cite{defossez2022a}. Let $d \in \mathbb{N}$ be the dimensionality of the model. Given a function $h:\mathbb{R}^d \rightarrow \mathbb{R}$ we define $\nabla_i h$ as the $i$'th component of the gradient. For the stochastic loss, we assume we have access to an oracle providing i.i.d samples $(f_t)_{t\in \mathbb{N}}$. $\bbE_{t-1}$ is defined as the expectation conditioned upon observing the past stochastic function values: $f_1, \ldots, f_{t-1}$. Unlike \cite{defossez2022a}, for simplicity of calculations, we define $\epsilon_t = \epsilon/t$ where $\epsilon$ is a small constant for stability. Next we define adaptive moment updates in a generic way, and then specialize them to AVGrad. These updates are different from those in Section~\ref{sec:dsp}, \ref{sec:avgrad}, but they are equivalent, and facilitates easier analysis.

\textbf{Adaptive updates.} Similar to \cite{defossez2022a} we define the following vectors iteratively. 
\begin{align*}
    & m_{t,i} = \beta_1 m_{t-1,i} + \nabla_i f_t(x_{t-1})\\
    & v_{t,i} = \beta_2 v_{t-1,i} + \nabla_i f_t(x_{t-1})^2
\end{align*}
Note that $\hat m_{t,i} = (1-\beta_1)m_{t,i}$ and $\hat v_{t,i} = (1-\beta_2)v_{t,i}$ give the updates in Adam. \cite{defossez2022a} uses the learning rate to absorb $(1-\beta)$ terms. In particular, if one has an update $x_{t,i} = x_{t-1,i} - \alpha_t\frac{m_{t,i}}{\sqrt{\epsilon+v_{t,i}}}$; Adam is recovered with $\alpha_t = \alpha \frac{1-\beta_1}{\sqrt{1-\beta_2}}$.

\textbf{AVGrad updates.} Let $\alpha_t = \frac{\alpha}{\sqrt{t}}$ be the learning rate. AVGrad is defined via the following iterative steps
\begin{align*}
    v^{(sum)}_{t,i} = v^{(sum)}_{t-1,i} + v_{t,i} =\sum_{j=1}^t v_{j,i} \text{, and } v^{(avg)}_{t,i} = \frac{v^{(sum)}_{t,i}}{t} \\
    x_{t,i} = x_{t-1,i} - \alpha_t\frac{m_{t,i}}{\sqrt{\epsilon_t+v^{(avg)}_{t,i}}} = x_{t-1,i} - \alpha \frac{m_{t,i}}{\sqrt{\epsilon+v^{(sum)}_{t,i}}}.
\end{align*}
Note that to obtain the AVGrad updates we also need to scale the learning rate with $\frac{1-\beta_1}{\sqrt{1-\beta_2}}$; but, for now, we assume it is already absorbed in the $\alpha$ for the sake of simplicity. We also drop the bias correction for simplicity as justified in  \cite{defossez2022a}.

\textbf{Precise Statement of Theorem~\ref{thm3}.}
Under the assumptions above and $\alpha_t=\frac{\alpha}{\sqrt{t}}$ for some $\alpha > 0$, and a constant $0\leq\rho\leq1$; when the second-order normalization term is in the form of the interpolation $\frac{\rho}{t}v_t+(1-\rho)v_t^{(avg)}$ we have:
\begin{align*}
    \frac{1}{T}\sum_{t=1}^T \|\nabla F(x_t) \|^2 & \leq \frac{2R\sqrt{\rho + (1-\rho)T}(F_0-F^*)}{\alpha T \sqrt{1-\beta_2}} \\
    & \quad + \frac{2R\sqrt{\rho + (1-\rho)T}d}{\sqrt{1-\beta_2}T}
    \left( 2R+\alpha L \right) \left( \ln\left(1 + \frac{R^2 (\rho + (1-\rho) T)}{\epsilon(1-\beta_2)}\right) + T \left[\ln\left(\frac{\rho}{\beta_2}\right)\right]_+ \right)
\end{align*}

\subsection{Proof of Theorem~\ref{thm3}}
We multiply the normalization term of the update with $t$ from the decaying learning rate; hence, there is a multiplying factor of $t$ difference in definition of interpolated term in Theorem~\ref{thm3} and $\bar \psi$ in this proof. Let us start by defining $\bar \psi_{t,i}$
\begin{align}
     \bar \psi_{t,i} = \rho v_{t,i} + (1-\rho)v^{(sum)}_{t,i}
\end{align}
We also define $\tilde \psi_{t,i}$:
\begin{align}
     \tilde \psi_{t,i} = \rho \tilde v_{t,i} + (1-\rho)\tilde v^{(sum)}_{t,i} & = \rho (\beta_2 v_{t-1,i} + \bbE_{t-1} \nabla_i f_t(x_{t-1})^2) + (1-\rho) (v^{(sum)}_{t-1,i}+\beta_2 v_{t-1,i} + \bbE_{t-1}\nabla_i f_t(x_{t-1})^2) \notag \\
     & = (1-\rho)v^{(sum)}_{t-1,i} + \beta_2 v_{t-1,i} + \bbE_{t-1}\nabla_i f_t(x_{t-1})^2, \label{tildepsi}
\end{align}
where the contribution of the last gradient is replaced by its expected value conditioned on the past iteration. Let us also define
\begin{equation*}
    u_{t,i} = \frac{\nabla_i f_t(x_{t-1})}{\sqrt{\epsilon + \overline \psi_{t,i}}}.
\end{equation*}
From the smoothness of the objective function $F$, we can use the Descent Lemma to obtain that
\begin{equation*}
    F(x_{t}) \leq F(x_{t-1}) - \alpha \nabla F(x_{t-1})^\top u_t+ \frac{\alpha^2 L}{2}\| u_t \|^2.
\end{equation*}

Taking the expectation of both sides conditioned on $f_1, \ldots, f_{t-1}$, we have
\begin{align} 
    \bbE_{t-1} F(x_{t}) \leq F(x_{t-1}) - \alpha \sum_{i \in [d]} \bbE_{t-1}\left[\nabla_i F(x_{t-1})\frac{\nabla_i f_t(x_{t-1})}{\sqrt{\epsilon + \bar \psi_{t,i} }} \right] + \frac{\alpha^2 L}{2} + \bbE_{t-1} [\|u_t\|^2]. \label{eq:cond exp}
\end{align}

To upper bound the second term on the right hand side, we use the following lemma which generalizes Lemma 5.1 in \citep{defossez2022a} to be used with interpolated optimizer and replaces $v_t$ with $\psi_{t,i}$ as the normalization term.

\begin{lemma}[Updates approximately follow a descent direction]\label{thm3:lemma1}
    \begin{align}
        \bbE_{t-1} \left[ \nabla_i F(x_{t-1}) \frac{\nabla_i f_t(x_{t-1})}{\sqrt{\epsilon+\bar \psi_{t,i}}} \right] \geq \frac{\nabla_i F(x_{t-1})^2}{2\sqrt{\epsilon + \tilde \psi_{t,i}}}- 2R\bbE_{t-1} \left[ \frac{\nabla_i f_t(x_{t-1})^2}{\epsilon + \psi_{t,i}} \right].
    \end{align}
\end{lemma}
\begin{proof}
    For the sake of simplicity, we define $G = \nabla_i F(x_{t-1}), \ g = \nabla_i f_t(x_{t-1}), \psi = \bar \psi_{t,i}, \tilde \psi = \tilde \psi_{t,i}$. First obviously, we have
    \begin{align} \label{lemma5: total}
        \frac{Gg}{\sqrt{\epsilon+\psi}} = \frac{Gg}{\sqrt{\epsilon+\tilde \psi}} + \underbrace{\frac{Gg}{\sqrt{\epsilon+\psi}} - \frac{Gg}{\sqrt{\epsilon+\tilde \psi}}}_A.
    \end{align}
    For the first term, we use the fact that $\bbE_{t-1}[g]=G$ to obtain that
    \begin{align}\label{lemma5:first term}
        \bbE_{t-1} \left[\frac{Gg}{\sqrt{\epsilon+\tilde \psi}} \right]  = \frac{G^2}{\sqrt{\epsilon+\tilde \psi}}.
    \end{align}
    In order to bound $A$, we begin with
    \begin{align*}
     A = Gg \frac{\tilde \psi - \psi}{\sqrt{\epsilon+\psi}\sqrt{\epsilon+\tilde \psi}(\sqrt{\epsilon+\psi}+\sqrt{\epsilon+\tilde \psi})} = Gg \frac{\bbE_{t-1}g^2 - g^2}{\sqrt{\epsilon+\psi}\sqrt{\epsilon+\tilde \psi}(\sqrt{\epsilon+\psi}+\sqrt{\epsilon+\tilde \psi})},
     \end{align*}
     where the last equality follows from the fact that $\tilde \psi - \psi = \bbE_{t-1}[g^2] - g^{2}$ (see \eqref{tildepsi} and the definition of $\psi_{t,i}$). Hence, from the triangle inequality we have that
     \begin{align*}
     |A| & \leq \underbrace{|Gg| \frac{\bbE_{t-1}g^2}{\sqrt{\epsilon+\psi}(\epsilon+\tilde \psi)}}_{A_1} + \underbrace{|Gg| \frac{g^2}{\sqrt{\epsilon+\tilde \psi}(\epsilon+ \psi)}}_{A_2},
    \end{align*}
    where in the inequality follows from the fact that $\sqrt{\epsilon+\psi}+\sqrt{\epsilon+\tilde \psi} \geq \max(\sqrt{\epsilon+ \psi}, \sqrt{\epsilon+\tilde \psi})$. A useful fact that will be used later is the following
    \begin{align}\label{fact1}
        \forall \lambda>0, \ x,y \in \bbR \quad xy \leq \frac{\lambda x^2}{2} + \frac{y^2}{2\lambda}.
    \end{align}
    To bound $A_1$, we use \eqref{fact1} with $\lambda = \frac{\sqrt{\epsilon+\tilde \psi}}{2}, x=\frac{|G|}{\sqrt{\epsilon + \tilde \psi}}, y = |g| \frac{ \bbE_{t-1} g^2}{2\sqrt{\epsilon+\tilde \psi} \sqrt{\epsilon+\psi}}$, which yields that
    \begin{align*}
    A_1 \leq \frac{G^2}{4\sqrt{\epsilon + \tilde \psi}} + \frac{g^2\bbE_{t-1} [g^2]^2}{(\epsilon+\tilde \psi)^{3/2}(\epsilon + \psi)}.
    \end{align*}
    Taking the expectation and noting $\epsilon + \tilde \psi \geq \bbE_{t-1} [g^2]$ ensures that
    \begin{align*}
        \bbE_{t-1} [A_1] \leq \frac{G^2}{4\sqrt{\epsilon+\tilde \psi}} + \frac{\bbE_{t-1} [g^2]}{\sqrt{\epsilon+\tilde \psi}} \bbE_{t-1}\left[\frac{g^2}{\epsilon+\psi} \right] \leq \frac{G^2}{4\sqrt{\epsilon+\tilde \psi}} + R \bbE_{t-1}\left[\frac{g^2}{\epsilon+\psi} \right]
    \end{align*}
    where the last inequality uses our boundedness assumption $\sqrt{\bbE_{t-1} [g^2]}\leq R $ and the fact that $\sqrt{\epsilon + \tilde \psi} \geq \sqrt{\bbE_{t-1} [g^2]}$. Similarly, for $A_{2}$, we use \eqref{fact1} with $\lambda = \frac{\sqrt{\epsilon+\tilde \psi}}{2 \bbE_{t-1}g^2}, x= \frac{|Gg|}{\sqrt{\epsilon+\tilde \psi}}, y= \frac{g^2}{\epsilon  + \psi}$, which gives us (we used similar bounding arguments)
    \begin{align*}
    \bbE_{t-1} [A_2] \leq \frac{G^2}{4 \sqrt{\epsilon + \tilde \psi}}+ R \frac{\bbE_{t-1} [g^2]}{\epsilon+\psi}.
    \end{align*}
    Combining both upper bounds we have,
    \begin{align*}
        \bbE_{t-1} [|A|] \leq  \frac{G^2}{2\sqrt{\epsilon + \tilde \psi}}+ 2R\frac{\bbE_{t-1} [g^2]}{\epsilon+\psi} ,
    \end{align*}
    which can be equivalently written as follows
    \begin{align*}
        \bbE_{t-1} [-|A|] \geq - \left(\frac{G^2}{2\sqrt{\epsilon + \tilde \psi}} + 2R\frac{\bbE_{t-1} [g^2]}{\epsilon+\psi} \right),
    \end{align*}
    Finally, combining further with \eqref{lemma5:first term} and substituting into \eqref{lemma5: total} gives us,
    \begin{align*}
        \bbE_{t-1} \left[\frac{Gg}{\sqrt{\epsilon+\psi}} \right] & \geq \frac{G^2}{\sqrt{\epsilon + \tilde \psi}} - \left(\frac{G^2}{2\sqrt{\epsilon + \tilde \psi}} + 2R\frac{\bbE_{t-1} [g^2]}{\epsilon+\psi}\right) = \frac{G^2}{2\sqrt{\epsilon + \tilde \psi}} - 2R\frac{\bbE_{t-1} [g^2]}{\epsilon+\psi},
    \end{align*}
    which proves the desired result. 
\end{proof}
Using Lemma~\ref{thm3:lemma1} in \eqref{eq:cond exp} yields that
\begin{align*}
    \bbE_{t-1} [F(x_{t})] \leq F(x_{t-1}) - \left(\frac{\alpha\sqrt{1-\beta_2}}{2R\sqrt{\rho + (1-\rho)t}}\|\nabla F(x_{t-1})\|^2 - 2\alpha R \bbE_{t-1} [\| u_t \|^2]\right)+ \frac{\alpha^2 L}{2}\bbE_{t-1} [\| u_t \|^2].
\end{align*}
Summing both sides for $t=1, \ldots ,T$, and noting $\sqrt{t} \leq \sqrt{T}$ implies that
\begin{align*}
    \bbE [F(x_{T})] \leq F(x_{0}) - \frac{\alpha \sqrt{1-\beta_2}}{2R\sqrt{\rho + (1-\rho)T}} \sum_{t=0}^{T-1} \|\nabla F(x_{t})\|^2+ \left(2\alpha R+\frac{\alpha^2 L}{2}  \right)\sum_{t=0}^{T-1} \bbE[\| u_t \|^2].
\end{align*}
Equivalently, we can write,
\begin{align}
    \frac{\alpha \sqrt{1-\beta_2}}{2R\sqrt{\rho + (1-\rho)T}}\sum_{t=0}^{T-1} \|\nabla F(x_{t})\|^2 \leq F(x_{0})-\bbE [F(x_{T})] + \left(2\alpha R+\frac{\alpha^2 L}{2}  \right)\sum_{t=0}^{T-1} \bbE[\| u_t \|^2]. \label{eq:after lemma 1}
\end{align}
Now, we would like to bound the last term on the right hand side. For this, we introduce the following lemma that generalizes Lemma 5.2 in \citep{defossez2022a}.

\begin{lemma}\label{thm3:lemma2}
    Define $b_t=\rho \sum_{j=1}^t \beta_2^{t-j}a_j + (1-\rho)\sum_{\tau=1}^t \sum_{j=1}^\tau \beta_2^{\tau-j} a_j$ for $t>0$ and $b_0 = 0$, we have
    \begin{align*}
        \sum_{t=1}^T\frac{a_t}{\epsilon+b_t} \leq  \ln\left(1 + \frac{R^2 (\rho + (1-\rho) T)}{\epsilon(1-\beta_2)}\right) + T \left[\ln\left(\frac{\rho}{\beta_2}\right)\right]_+.
    \end{align*}
\end{lemma}
\begin{proof}
    Let $l_t = \sum_{j=1}^t \beta_2^{t-j}a_j $, $r_t = \sum_{\tau=1}^t \sum_{j=1}^\tau \beta_2^{\tau-j} a_j $, we have $b_t = \rho l_t + (1-\rho)r_t$. Since $b_t > a_t \geq 0 $ we have that $1-z \leq \exp^{-z} $ with $z = \frac{a_t}{\epsilon + b_t} < 1$. Hence, 
    \begin{align*}
        \frac{a_t}{\epsilon + b_t} & \leq \ln(\epsilon+b_t) - \ln(\epsilon + b_t - a_t)\\
        & = \ln(\epsilon+b_t) - \ln(\epsilon + \rho (l_t - a_t) + (1-\rho)(r_t - a_t)) \\
        & = \ln(\epsilon+b_t) - \ln\Big(\epsilon + \rho \beta_2l_{t-1} + (1-\rho)r_{t-1} + (1-\rho)(\sum_{j =1}^{t} \beta_2^{t-j} a_j - a_t)\Big) \\
        & = \ln(\epsilon+b_t) - \ln\Big(\epsilon + \rho \beta_2l_{t-1} + (1-\rho)r_{t-1} + (1-\rho)\beta_2 \underbrace{\sum_{j =1}^{t-1} \beta_2^{t-1-j} a_j}_{l_{t-1}}\Big) \\
        & = \ln(\epsilon+b_t) - \ln\left(\epsilon +\beta_2l_{t-1} + (1-\rho)r_{t-1} \right) \\
        & = \ln\left(\frac{\epsilon+b_t}{\epsilon+b_{t-1}}\right) + \ln\left(\frac{\epsilon+\rho l_{t-1} + (1-\rho) r_{t-1} }{\epsilon + \beta_2l_{t-1} + (1-\rho)r_{t-1}}\right)\\
        & \leq \ln\left(\frac{\epsilon+b_t}{\epsilon+b_{t-1}}\right) + \ln\left(\max\{1,\frac{\rho}{\beta_2}\}\right) \\
        & = \ln\left(\frac{\epsilon+b_t}{\epsilon+b_{t-1}}\right) + \left[\ln\left(\frac{\rho}{\beta_2}\right)\right]_+,
    \end{align*}
    where the first equality is due to definition of $b_t$. Summing the inequality above for $t = 1,2, \ldots, T$, yields that (recall that $b_0 = 0$)
    \begin{align*}
        \sum_{t=1}^T\frac{a_t}{\epsilon + b_t} \leq \ln\left(1 + \frac{b_T}{\epsilon}\right) + T \left[\ln\left(\frac{\rho}{\beta_2}\right)\right]_+
    \end{align*}
    Also note that $b_T \leq \frac{R^2 (\rho + (1-\rho) T)}{1-\beta_2}$.
\end{proof}
Using Lemma~\ref{thm3:lemma2} in \eqref{eq:after lemma 1}, dividing both sides by $T$, and after some algebra we have
\begin{align*}
    \frac{1}{T}\sum_{t=1}^T \|\nabla F(x_t) \|^2 & \leq \frac{2R\sqrt{\rho + (1-\rho)T}(F_0-F^*)}{\alpha T \sqrt{1-\beta_2}} \\
    & \quad + \frac{2R\sqrt{\rho + (1-\rho)T}d}{\sqrt{1-\beta_2}T}
    \left( 2R+\alpha L \right) \left( \ln\left(1 + \frac{R^2 (\rho + (1-\rho) T)}{\epsilon(1-\beta_2)}\right) + T \left[\ln\left(\frac{\rho}{\beta_2}\right)\right]_+ \right)
\end{align*}
which concludes the proof.


\section{Convergence of AVGrad} \label{app:avgrad}

Using the notation and assumptions in Section~\ref{sec:mada}, we give the following convergence results for AVGrad. 

\begin{theorem}[Convergence of AVGrad without momentum] \label{thm1} Under the above assumptions and $\alpha_t=\frac{\alpha}{\sqrt{t}}$ for some $\alpha > 0$, we have:

\begin{align*}
    \frac{1}{T}\sum_{t=1}^T \|\nabla F(x_t) \|^2 \leq \frac{2R(F_0-F^*)}{\alpha \sqrt{T} \sqrt{1-\beta_2}} + \frac{2Rd}{\sqrt{T}\sqrt{1-\beta_2}}\left( 2R+\alpha L \right) \ln\left(1+\frac{R^2T}{(1-\beta_2)\epsilon}\right).
\end{align*}

\end{theorem}

\begin{theorem}[Convergence of AVGrad with momentum]\label{thm2}Let $\tau_T \in \{0, \ldots , T-1\}$ denote a random index such that $\forall j \in \mathbb{N}, j<T, \mathbb{P}[\tau=j] \propto 1-\beta_1^{T-j}$. Under the above assumptions and $\alpha_t=\frac{\alpha}{\sqrt{t}}$ for some $\alpha > 0$, we have:
\begin{align*}
    \bbE \|\nabla F(x_\tau) \|^2 \leq \frac{2(1-\beta_1)R\sqrt{T}}{\alpha \sqrt{1-\beta_2} \tilde T}(F(x_0)-F^*) + C\frac{\sqrt{T}d}{\tilde{T}} \ln\left(1+\frac{R^2T}{(1-\beta_2)\epsilon}\right),
\end{align*}
where $C=\frac{\alpha RL}{ \sqrt{1-\beta_2}(1-\beta_1)} + \frac{2\beta_1\alpha^2L^2}{(1-\beta_2)(1-\beta_1)^3} + \frac{12R^2}{\sqrt{1-\beta_1}} \text{ and } \tilde{T} = T - \frac{\beta_1}{1-\beta_1}$.
\end{theorem}

\textbf{Remark.} Comparing our Theorems~\ref{thm1} and \ref{thm2} to Theorems 3 and 4 in \cite{defossez2022a}, we observe that AVGrad does not have the non-vanishing, constant term that Adam has (see \citep{defossez2022a}, Theorem 2). Moreover, unlike the result for Adam, we do not require $\beta_2 > \beta_1$. Analogous to how momentum slows down the convergence bound of algorithms (by multiplicative factors) \citep{defossez2022a} AVGrad slows down the convergence of Adagrad by multiplicative factors of $(1-\beta_2)$.

\subsection{Proof of Theorem \ref{thm1}}

Proof of convergence of AVGrad can be obtained by replacing the interpolated second order moment term $\bar \psi_t$, by $v_{t}^{(sum)}$ in the previous section. In other words, setting $\rho=0$ in the Proof of Theorem~\ref{thm3}, going through the analysis will give the desired result
\begin{align*}
    \frac{1}{T}\sum_{t=1}^T \|\nabla F(x_t) \|^2 \leq \frac{2R(F_0-F^*)}{\alpha \sqrt{T} \sqrt{1-\beta_2}} + \frac{2Rd}{\sqrt{1-\beta_2}\sqrt{T}}
    \left( 2R+\alpha L \right) \ln\left(1+\frac{TR^2}{(1-\beta_2)\epsilon}\right).
\end{align*}

\subsection{Proof of Theorem \ref{thm2} } 
The idea in this proof is to essentially change gradient terms in the Descent Lemma with first order moments, as the difference in consequent model weights will now depend on the moment term rather than a gradient at some time point. The necessary changes in the lemmas closely follow the proof for Theorem 3,4 in \citep{defossez2022a}. We start by redefining some iterative vectors. 

\begin{align*}
    & m_{t,i} = \beta_1 m_{t-1,i} + \nabla_i f_t(x_{t-1})\\
    & v_{t,i} = \beta_2 v_{t-1,i} + \nabla_i f_t(x_{t-1})^2\\
    & x_{t,i} = x_{t-1,i} - \alpha_t\frac{m_{t,i}}{\sqrt{\epsilon_t+v^{(avg)}_{t,i}}} = x_{t-1,i} - \alpha\frac{m_{t,i}}{\sqrt{\epsilon+v^{(sum)}_{t,i}}}
\end{align*}

Let us further define $G_t = \nabla F(x_{t-1}), g_t = \nabla f_t(x_{t-1}), u_{t,i}= \frac{m_{t,i}}{\sqrt{\epsilon+v^{(sum)}_{t,i}}}$ and $U_{t,i}= \frac{g_{t,i}}{\sqrt{\epsilon+v^{(sum)}_{t,i}}}$. And also define:
\begin{align*}
    \tilde v^{(sum)}_{t,k,i} = v^{(sum)}_{t-k,i} + \bbE_{t-k-1} [\sum_{\tau=t-k+1}^t \sum_{j=1}^\tau \beta_2^{\tau-j} g^2_{j,i}]
\end{align*}
note that for $j\leq t-k$ we have $\bbE_{t-k-1} [g^2_j] = g^2_j$, so $\tilde v^{(sum)}_{t,k,i}$ essentially replaces the contribution of last $k$ gradients with their expected values. From the smoothness of the objective function $F$, we have

\begin{align*}
    F(x_t) \leq F(x_{t-1}) - \alpha G^\top_t u_t + \frac{\alpha^2L}{2}\|u_t\|^2_2.
\end{align*}

Taking the expectations of both sides,
\begin{align}\label{eq:expected smoothness}
   \bbE [F(x_t)] \leq \bbE [F(x_{t-1})] - \alpha \sum_{i \in [d]} \bbE [G^\top_{t,i} u_{t,i}] + \frac{\alpha^2L}{2} \bbE [\|u_t\|^2_2].
\end{align}

To bound the second term on the right hand side, we introduce the following approximate descent lemma, whose proof is provided at the end of section.

\begin{lemma} \label{lemma3} [Updates approximately follow a descent direction] 
\begin{align*}
    \bbE \left[ G_{t,i} \frac{m_{t,i}}{\sqrt{\epsilon+v^{(sum)}_{t,i}}} \right] & \geq \frac{1}{2} \left( \sum_{i \in [d]}\sum_{k=0}^{t-1} \beta_1^k \bbE \left[ \frac{G^2_{t-k,i}}{\sqrt{\epsilon+\tilde v^{(sum)}_{t,k+1,i}}} \right] \right) - \frac{3R\sqrt{1-\beta_2}}{\sqrt{1-\beta_1}} \sum_{k=0}^{t-1} \sqrt{k+1} \beta_1^k \bbE[\| U_{t-k} \|^2] \notag \\
    & \quad - \frac{\alpha^2L^2}{4R\sqrt{1-\beta_2}}\sqrt{1-\beta_1} \sum_{l=1}^{t-1}\| u_{t-l}\|^2 \sum_{k=l}^{t-1} \beta_1^k \sqrt{k}.
    \end{align*}
\end{lemma}

Using Lemma~\ref{lemma3} in \eqref{eq:expected smoothness} for the second term on right hand side, yields that
\begin{align}
    \bbE [F(x_t)] & \leq \bbE [F(x_{t-1})] - \frac{\alpha}{2}\left( \sum_{i \in [d]} \sum_{k=0}^{t-1} \beta_1^k \bbE \left[ \frac{G^2_{t-k,i}}{\sqrt{\epsilon + \tilde v^{(sum)}_{t,k+1,i}}} \right] \right) + \frac{3\alpha R\sqrt{1-\beta_2}}{\sqrt{1-\beta_1}} \sum_{k=0}^{t-1} \sqrt{k+1} \beta_1^k \bbE [\|U_{t-k}\|^2_2] \notag  \\
    & \quad + \frac{\alpha^3L^2}{4R\sqrt{1-\beta_2}} \sqrt{1-\beta_1} \sum_{l=1}^{t-1}\|u_{t-l}\|^2_2 \sum_{k=l}^{t-1} \beta_1^k \sqrt{k}  + \frac{\alpha^2L}{2}\bbE [\|u_t\|^2_2]. \label{eq:after l3}
\end{align}
Let us define $\Omega_t := \sqrt{\sum_{\tau=1}^t \sum_{j=1}^\tau \beta_2^{\tau-j}}$, hence, from our boundedness assumption, $\sqrt{\epsilon + \tilde v^{(sum)}_{t,k+1,i}} \leq R \sqrt{\sum_{\tau=1}^t \sum_{j=1}^\tau \beta_2^{\tau-j}} = R\Omega_t$, inserting in \eqref{eq:after l3} implies that 
\begin{align*}
    \bbE [F(x_t)] & \leq \bbE [F(x_{t-1})] - \frac{\alpha}{2R\Omega_t}\left( \sum_{k=0}^{t-1} \beta_1^k \bbE \left[ G^2_{t-k,i} \right] \right) + \frac{\alpha^2L}{2}\bbE \|u_t\|^2_2 + \frac{3\alpha R\sqrt{1-\beta_2}}{\sqrt{1-\beta_1}} \sum_{k=0}^{t-1} \sqrt{k+1} \beta_1^k \bbE \|U_{t-k}\|^2_2 \notag \\
    & \quad + \frac{\alpha^3L^2}{4R\sqrt{1-\beta_2}} \sqrt{1-\beta_1} \sum_{l=1}^{t-1}\|u_{t-l}\|^2_2 \sum_{k=l}^{t-1} \beta_1^k \sqrt{k}
\end{align*}

Summing for $t=1, \ldots, T$ results in
\begin{align}
    \underbrace{\frac{\alpha}{2R}\sum_{t=1}^T \frac{1}{\Omega_t}\sum_{k = 0}^{t-1}\beta_1^k \bbE \|G^2_{t-k}\|^2_2 }_A & \leq F(x_0)-F^* + \underbrace{\frac{\alpha^2L}{2} \sum_{t=1}^T \bbE [\|u_t\|^2]}_B+\underbrace{\frac{\alpha^3L^2}{4R\sqrt{1-\beta_2}} \sqrt{1-\beta_1} \sum_{t=1}^T \sum_{l=1}^{t-1}\bbE[\|u_{t-l}\|^2] \sum_{k=l}^{t-1} \beta_1^k \sqrt{k}}_C \notag \\
    & \quad + \underbrace{\frac{3\alpha R\sqrt{1-\beta_2}}{\sqrt{1-\beta_1}} \sum_{t=1}^T \sum_{k=0}^{t-1} \sqrt{k+1} \beta_1^k \bbE [\|U_{t-k}\|^2]}_D. \label{eq:all_terms}
\end{align}

We will examine each term separately, for $B$, we state the following lemma (whose proof is given at the end of section) to bound the sum of squared norm term

\begin{lemma}\label{lemma4}
    Assume, $0 \leq \beta_1 < 1$, $0 \leq \beta_2 < 1$ and we have sequence of real numbers $(a_t)_{t\in[T]}$. Let $c_t = \sum_{\tau = 1}^t \beta_1^{t-\tau}a_t$, $b_t = \sum_{\tau = 1}^t\sum_{j=1}^\tau \beta_2^{\tau-j}a^2_t$. Then, 
    \begin{align*}
        \sum_{t=1}^T \frac{c_t^2}{\epsilon+b_t} \leq \frac{1}{(1-\beta_1)^2}\ln{\left( 1+\frac{b_T}{\epsilon}\right)}.
    \end{align*}
\end{lemma}

Lemma~\ref{lemma4} implies that 
\begin{align}\label{prf2:B}
    \frac{\alpha^2L}{2} \sum_{t=1}^T \bbE [\|u_t\|^2] \leq \frac{\alpha^2L}{2(1-\beta_1)^2}\sum_{i \in [d]}\ln \left( 1+\frac{v^{(sum)}_{T,i}}{\epsilon} \right) \leq \frac{\alpha^2L}{2(1-\beta_1)^2}\sum_{i \in [d]}\ln \left( 1+\frac{R^2T}{(1-\beta_2)\epsilon} \right).
\end{align}

Before moving on with other terms, we state some useful facts that are proven in \cite{defossez2022a}.
\begin{fact}\label{fact:q}
    Given $0<a<1$ and $Q \in \mathbb{N}$ we have,
    \begin{align*}
        \sum_{q=0}^{Q-1}a^q q \leq \frac{a}{(1-a)^2}.
    \end{align*}
\end{fact}

\begin{fact}\label{fact:sqq}
    Given $0<a<1$ and $Q \in \mathbb{N}$ we have,
    \begin{align*}
        \sum_{q=0}^{Q-1}a^q \sqrt{q} \leq \frac{2}{(1-a)^{3/2}}.
    \end{align*}
\end{fact}

\begin{fact}\label{fact:sqqq}
    Given $0<a<1$ and $Q \in \mathbb{N}$ we have,
    \begin{align*}
        \sum_{q=0}^{Q-1}a^q \sqrt{q} (q+1) \leq \frac{4a}{(1-a)^{5/2}}.
    \end{align*}
\end{fact}

Now we move onto examining other terms in \eqref{eq:all_terms}. For $C$, we make the following change in the index $j = t-l$ which yields the following steps

\begin{align}\label{prf2:C}
    \frac{\alpha^3L^2}{4R\sqrt{1-\beta_2}} \sqrt{1-\beta_1} \sum_{t=1}^T \sum_{l=1}^{t-1}\bbE[\|u_{t-l}\|^2] \sum_{k=l}^{t-1} \beta_1^k \sqrt{k} & = \frac{\alpha^3L^2}{4R\sqrt{1-\beta_2}} \sqrt{1-\beta_1} \sum_{t=1}^T \sum_{j=1}^{t}\bbE[\|u_{j}\|^2] \sum_{k=t-j}^{t-1} \beta_1^k \sqrt{k} \notag \\
    & = \frac{\alpha^3L^2}{4R\sqrt{1-\beta_2}} \sqrt{1-\beta_1} \sum_{j=1}^T \bbE[\|u_{j}\|^2] \sum_{t=j}^{T} \sum_{k=t-j}^{t-1} \beta_1^k \sqrt{k} \notag \\
    & = \frac{\alpha^3L^2}{4R\sqrt{1-\beta_2}} \sqrt{1-\beta_1} \sum_{j=1}^T \bbE[\|u_{j}\|^2] \sum_{k=0}^{T-1} \beta_1^k \sqrt{k} \sum_{t=j}^{j+k}  1 \notag \\
    & = \frac{\alpha^3L^2}{4R\sqrt{1-\beta_2}} \sqrt{1-\beta_1} \sum_{j=1}^T \bbE[\|u_{j}\|^2] \sum_{k=0}^{T-1} \beta_1^k \sqrt{k}(k+1)\notag \\
    & \stackrel{(a)}{\leq} \frac{\alpha^3L^2}{R\sqrt{1-\beta_2}} \sum_{j=1}^T \bbE[\|u_{j}\|^2] \frac{\beta_1}{(1-\beta_1)^2}\notag \\
    & \stackrel{(b)}{\leq} \frac{\alpha^3L^2}{R\sqrt{1-\beta_2}} \frac{\beta_1}{(1-\beta_1)^4} \sum_{i\in [d]} \ln \left( 1+\frac{v^{(sum)}_{T,i}}{\epsilon} \right) \\
    & \stackrel{(c)}{\leq} \frac{\alpha^3L^2}{R\sqrt{1-\beta_2}} \frac{\beta_1}{(1-\beta_1)^4} \sum_{i\in [d]} \ln \left( 1+\frac{R^2T}{(1-\beta_2)\epsilon} \right),
\end{align}
where in (a) we use Fact~\ref{fact:sqqq}, in (b) we use Lemma~\ref{lemma4}, and in (c) the definition of $v_{T,i}^{(sum)}$. For $D$, we make the following change in the index $j = t-k$, and obtain that
\begin{align}\label{prf2:D}
    \frac{3\alpha R\sqrt{1-\beta_2}}{\sqrt{1-\beta_1}} \sum_{t=1}^T \sum_{k=0}^{t-1} \sqrt{k+1} \beta_1^k \bbE [\|U_{t-k}\|^2] & = \frac{3\alpha R\sqrt{1-\beta_2}}{\sqrt{1-\beta_1}} \sum_{t=1}^T \sum_{j=1}^{t} \sqrt{1+t-j} \beta_1^{t-j} \bbE [\|U_{j}\|^2] \notag\\
    & = \frac{3\alpha R\sqrt{1-\beta_2}}{\sqrt{1-\beta_1}} \sum_{j=1}^T \bbE [\|U_{j}\|^2] \sum_{t=j}^{T} \sqrt{1+t-j} \beta_1^{t-j} \notag\\
    & \stackrel{(a)}{\leq} \frac{3\alpha R\sqrt{1-\beta_2}}{\sqrt{1-\beta_1}} \sum_{j=1}^T \bbE [\|U_{j}\|^2] \frac{2}{(1-\beta_1)^{3/2}} \notag\\
    & \stackrel{(b)}{\leq}  \frac{6\alpha R\sqrt{1-\beta_2} }{(1-\beta_1)^{2}} \sum_{i \in [d]} \ln \left( 1+\frac{v^{(sum)}_{T,i}}{\epsilon} \right) \\
    & \stackrel{(b)}{\leq}  \frac{6\alpha R\sqrt{1-\beta_2} }{(1-\beta_1)^{2}} \sum_{i \in [d]} \ln \left( 1+\frac{R^2}{(1-\beta_2)\epsilon} \right) ,
\end{align}
where in (a) we use Fact~\ref{fact:sqq} and in (b) we use Lemma~\ref{thm3:lemma2}. For $A$, first note that,
\begin{align}\label{omega ineq}
    \Omega_t = \sqrt{\sum_{\tau=1}^t \sum_{j=1}^\tau \beta_2^{\tau-j}} \leq \sqrt{\frac{t}{(1-\beta_2)}} \leq \sqrt{\frac{T}{(1-\beta_2)}} .
\end{align}
Let us change index, $j=t-k$, and use \eqref{omega ineq}:
\begin{align*}
    \frac{\alpha}{2R}\sum_{t=1}^T \frac{1}{\Omega_t}\sum_{k = 0}^{t-1}\beta_1^k \bbE [\|G^2_{t-k}\|^2]
    & \geq \frac{\alpha \sqrt{1-\beta_2}}{2R\sqrt{T}} \sum_{t=1}^T \sum_{j=1}^t \beta_1^{t-j} \bbE [\|G_j\|^2] \\
    & = \frac{\alpha \sqrt{1-\beta_2}}{2R\sqrt{T}} \sum_{j=1}^T  \bbE [\|G_j\|^2] \sum_{t=j}^T \beta_1^{t-j} \\
    & = \frac{\alpha \sqrt{1-\beta_2}}{2(1-\beta_1)R\sqrt{T}} \sum_{j=1}^T (1-\beta_1^{T-j+1})  \bbE [\|G_j\|^2] \\
    & = \frac{\alpha \sqrt{1-\beta_2}}{2(1-\beta_1)R\sqrt{T}} \sum_{j=0}^{T-1} (1-\beta_1^{T-j})  \bbE [\|\nabla F(x_j)\|^2].
\end{align*}
Note, $\sum_{j=0}^{T-1}(1-\beta_1^{T-j})=T-\beta_1 \frac{1-\beta_1^T}{1-\beta_1}\geq T - \frac{\beta_1}{1-\beta_1}=\tilde T$, and let $\tau \in \{0, \ldots, T-1\}$ and $\mathbb{P}[\tau = j] \propto 1-\beta_1^{T-j}$, then we have:
\begin{align}\label{prf2:A}
    A \geq \frac{\alpha \sqrt{1-\beta_2}}{2(1-\beta_1)R\sqrt{T}} \tilde T \bbE [\|\nabla F(x_\tau)\|^2].
\end{align}
Inserting \eqref{prf2:B}, \eqref{prf2:C}, \eqref{prf2:D}, \eqref{prf2:A} in \eqref{eq:all_terms} yields that

\begin{align*}
    \underbrace{\frac{\alpha \sqrt{1-\beta_2}}{2(1-\beta_1)R\sqrt{T}} \tilde T \bbE [\|\nabla F(x_\tau)\|^2]}_A & \leq F(x_0)-F^* + \underbrace{\frac{\alpha^2L}{2}\frac{1}{(1-\beta_1)^2}\sum_{i \in [d]}\ln \left( 1+\frac{R^2T}{(1-\beta_2)\epsilon} \right)}_B \\ & \quad+\underbrace{\frac{\alpha^3L^2}{R\sqrt{1-\beta_2}} \frac{\beta_1}{(1-\beta_1)^4} \sum_{i\in [d]} \ln \left( 1+\frac{R^2T}{(1-\beta_2)\epsilon} \right)}_C \notag
    + \underbrace{\frac{6\alpha R\sqrt{1-\beta_2} }{(1-\beta_1)^{2}} \sum_{i \in [d]} \ln \left( 1+\frac{R^2}{(1-\beta_2)\epsilon} \right) }_D.
\end{align*}

and after some algebra we have,
\begin{align*}
\bbE [\|\nabla F(x_\tau)\|^2] & \leq \frac{2(1-\beta_1)R\sqrt{T}}{\alpha \sqrt{1-\beta_2} \tilde T}(F(x_0)-F^*) + \frac{\alpha RL\sqrt{T}  d}{ \sqrt{1-\beta_2}(1-\beta_1) \tilde T} \ln \left( 1 + \frac{R^2T}{(1-\beta_2)^2\epsilon} \right) \notag \\
& \quad + \frac{2\beta_1\alpha^2L^2\sqrt{T}d}{(1-\beta_2)(1-\beta_1)^3\tilde T} \ln \left( 1 + \frac{R^2T}{(1-\beta_2)\epsilon} \right) + \frac{12R^2\sqrt{T}d}{\sqrt{1-\beta_1}\tilde T} \ln \left( 1 + \frac{R^2T}{(1-\beta_2)\epsilon} \right),
\end{align*}
Equivalently, 
\begin{align*}
    \bbE \|\nabla F(x_\tau) \|^2 \leq \frac{2(1-\beta_1)R\sqrt{T}}{\alpha \sqrt{1-\beta_2} \tilde T}(F(x_0)-F^*) + C\frac{\sqrt{T}d}{\tilde{T}} \ln\left(1+\frac{R^2T}{(1-\beta_2)\epsilon}\right),
\end{align*}
where $C=\frac{\alpha RL}{ \sqrt{1-\beta_2}(1-\beta_1)} + \frac{2\beta_1\alpha^2L^2}{(1-\beta_2)(1-\beta_1)^3} + \frac{12R^2}{\sqrt{1-\beta_1}}$ 
which concludes the proof.

\begin{proof}[Proof of \cref{lemma3}]
    We start by separating the main term into two:
    \begin{align}\label{lemma3:first term}
        G_{t,i}\frac{m_{t,i}}{\sqrt{\epsilon + v^{(sum)}_{t,i}}} &= \sum_{k=0}^{t-1} G_{t,i} \beta_1^{t-k}\frac{g_{t-k,i}}{\sqrt{\epsilon+v^{(sum)}_{t,i}}} \\
        & =  \underbrace{\sum_{k=0}^{t-1} G_{t-k,i} \beta_1^{t-k}\frac{g_{t-k,i}}{\sqrt{\epsilon+v^{(sum)}_{t,i}}}}_{A} + \underbrace{\sum_{k=0}^{t-1} (G_{t,i}-G_{t-k,i}) \beta_1^{t-k}\frac{g_{t-k,i}}{\sqrt{\epsilon+v^{(sum)}_{t,i}}}}_{B}.
    \end{align}
    For $B$, we again utilize the fact \eqref{fact1} that $\forall \lambda>0, \ x,y \in \bbR \quad xy \leq \frac{\lambda x^2}{2} + \frac{y^2}{2\lambda}$, for each dimension with $\lambda = \frac{\sqrt{1-\beta_1}}{2R\sqrt{1-\beta_2}\sqrt{k+1}}, x=|G_{t,i}-G_{t-k,i}|, y= \frac{|g_{t-k,i}|}{\sqrt{\epsilon+v^{(sum)}_{t,i}}}$. Then, 
    \begin{align*}
        |B| \leq \sum_{i \in [d]} \sum_{k=0}^{t-1} \beta_1^t\left( \frac{\sqrt{1-\beta_1}}{4R\sqrt{1-\beta_2}\sqrt{k+1}}(G_{t,i}-G_{t-k,i})^2 +\frac{2R\sqrt{1-\beta_2}\sqrt{k+1}g^2_{t-k,i}}{\sqrt{1-\beta_1}(\epsilon+v^{(sum)}_{t-k,i})} \right).
    \end{align*}
    Note that $\epsilon+v^{(sum)}_{t,i} \geq \epsilon+v^{(sum)}_{t-k,i}$, hence,
    \begin{align*}
        \frac{g^2_{t-k,i}}{\epsilon+v^{(sum)}_{t,i}}\leq  \frac{g^2_{t-k,i}}{\epsilon+v^{(sum)}_{t-k,i}} = U^2_{t-k,i}.
    \end{align*}
    Moreover, smoothness of objective function implies
    \begin{align*}
        \|G_t - G_{t-k}\|^2 \leq L^2\|x_{t-1} - x_{t-k-1}\|^2 = L^2\| \sum_{l=1}^k \alpha u_{t-l} \|^2 \leq \alpha^2 L^2 k \sum_{l=1}^k \| u_{t-l} \|^2
    \end{align*}
    As a result,
    \begin{align}\label{lemma3:B}
        |B| & \leq \sum_{k=0}^{t-1} \frac{\alpha^2L^2\beta_1^k\sqrt{1-\beta_1}\sqrt{k}}{4R\sqrt{1-\beta_2}} \sum_{l=1}^{k} \|u_{t-l}\|^2+ \sum_{k=0}^{t-1} \frac{R\sqrt{1-\beta_2}\sqrt{k+1}}{\sqrt{1-\beta_1}}\beta_1^k\|U_{t-k}\|^2 \notag \\
        & = \frac{\alpha^2 L^2}{4R\sqrt{1-\beta_2}} \sqrt{1-\beta_1} \sum_{l=1}^{t-1}\|u_{t-l}\|^2\sum_{k=l}^{t-1}\beta_1^k\sqrt{k} + \frac{R\sqrt{1-\beta_2}}{\sqrt{1-\beta_1}}\sum_{k=0}^{t-1}\sqrt{k+1}\beta_1^k\|U_{t-k}\|^2.
    \end{align}
    For term $A$, we focus on the main term of the summation $\bbE \left[G_{t-k} \frac{g_{t-k,i}}{\sqrt{\epsilon+v^{(sum)}_{t,i}}} \right]$. Let $G=G_{t-k,i}, g=g_{t-k,i}, \tilde v = \tilde v^{(sum)}_{t,k+1,i}, v = v^{(sum)}_{t,i}$.
    Note that,
    \begin{align*}
        \tilde v - v & = v^{(sum)}_{t-k,i} +\bbE_{t-k-1} \Big[ \sum_{\tau = t-k}^t \sum_{j=1}^\tau \beta_2^{\tau-j} g_j^2\Big] - v^{(sum)}_{t,i}\\
        & = \bbE_{t-k-1} [\sum_{\tau = t-k}^t \sum_{j=1}^\tau \beta_2^{\tau-j} g_j^2] - \sum_{\tau = t-k}^t v_{\tau} \\
        & = \underbrace{\bbE_{t-k-1} [\sum_{\tau = t-k}^t \sum_{j=1}^\tau \beta_2^{\tau-j} g_j^2}_{A_1^2}] - \underbrace{ \sum_{\tau = t-k}^t \sum_{j=1}^\tau \beta_2^{\tau-j} g_j^2}_{A_2^2}.
    \end{align*}
    We continue similar to Lemma~\ref{thm3:lemma1}.
    \begin{align}\label{lemma3:separation}
        \frac{Gg}{\sqrt{\epsilon + v}} = \frac{G^2}{\sqrt{\epsilon + \tilde v}} + \underbrace{Gg \frac{A_1^2 - A_2^2}{\sqrt{\epsilon+v} \sqrt{\epsilon + \tilde v} (\sqrt{\epsilon + v}+\sqrt{\epsilon + \tilde v})}}_C.
    \end{align}
    We have
    \begin{align*}
        |C| \leq \underbrace{\frac{|Gg|A_1^2}{\sqrt{\epsilon+v}(\epsilon+\tilde v)}}_{C_1}+\underbrace{\frac{|Gg| A_2^2}{(\epsilon+v)\sqrt{\epsilon+\tilde v}}}_{C_2}.
    \end{align*}
    We will utilize \eqref{fact1}, for $C_1$, let $\lambda = \frac{\sqrt{(1-\beta_1)}\sqrt{\epsilon + \tilde v}}{2}$, $x = \frac{|G|}{\sqrt{\epsilon+\tilde v}}$, $y = \frac{|g|A_1^2}{\sqrt{\epsilon+\tilde v} \sqrt{\epsilon+v}}$. Then,
    \begin{align}
        C_1 \leq \frac{G^2}{4\sqrt{\epsilon+\tilde v}}+ \frac{1}{\sqrt{1-\beta_1}}\frac{g^2A_1^4}{(\epsilon+ \tilde v)^{\frac{3}{2}}(\epsilon+v)}.
    \end{align}
    Given $\epsilon + \tilde v \geq A_1^2$ and taking the expectation:
    \begin{align*}
        \bbE_{t-k-1} [C_1] \leq \frac{G^2}{4\sqrt{\epsilon+\tilde v}}+ \frac{1}{\sqrt{1-\beta_1}}\frac{A_1^2}{\sqrt{\epsilon + \tilde v}}\bbE_{t-k-1}\left[\frac{g^2}{(\epsilon+v)}\right].
    \end{align*}
    Similarly, for $C_2$, we let $
    \lambda = \frac{\sqrt{1-\beta_1}\sqrt{\epsilon+\tilde v}}{2A_1^2}$, $x=\frac{|GA_2|}{\sqrt{\epsilon + \tilde v}}$, $y=\frac{|A_2 g|}{\epsilon+v}$. Then,
    \begin{align*}
        C_2 \leq \frac{G^2}{4\sqrt{\epsilon + \tilde v}}\frac{A_2^2}{A_1^2} + \frac{1}{\sqrt{1-\beta_1}}\frac{A_1^2}{\sqrt{\epsilon + \tilde v}}\frac{g^2 A_2^2}{(\epsilon + v)^2}.
    \end{align*}
    Using $\epsilon+v\geq A_2^2$ and $\bbE_{t-k-1}[\frac{A_2^2}{A_1^2}]=1$,
    \begin{align*}
        \bbE_{t-k-1} [C_2] \leq \frac{G^2}{4\sqrt{\epsilon + \tilde v}} + \frac{1}{\sqrt{1-\beta_1}}\frac{A_1^2}{\sqrt{\epsilon + \tilde v}}\bbE_{t-k-1}\left[\frac{g^2}{(\epsilon + v)}\right].
    \end{align*}
    Hence, 
    \begin{align*}
        \bbE_{t-k-1} [|C|] \leq \frac{G^2}{2\sqrt{\epsilon + \tilde v}} + \frac{1}{\sqrt{1-\beta_1}}\frac{2A_1^2}{\sqrt{\epsilon + \tilde v}}\bbE_{t-k-1}\left[\frac{g^2}{(\epsilon + v)}\right].
    \end{align*}
    Using $A_1 \leq \sqrt{\epsilon+\tilde v}$, thus, $A_1 \leq R\sqrt{k+1}\sqrt{1-\beta_2}$:
    \begin{align*}
        \bbE_{t-k-1} [|C|] \leq \frac{G^2}{2\sqrt{\epsilon + \tilde v}} + \frac{2R\sqrt{k+1}\sqrt{1-\beta_2}}{\sqrt{1-\beta_1}}\bbE_{t-k-1}\left[\frac{g^2}{(\epsilon + v)}\right].
    \end{align*}
    Taking complete expectation, using $\epsilon+v^{(sum)}_{t,i} \geq \epsilon+v^{(sum)}_{t-k,i}$ and reintroducing the indices:
    \begin{align*} 
        \bbE [|C|] \leq \frac{1}{2}\bbE \left[\frac{G_{t-k,i}^2}{\sqrt{\epsilon + \tilde v^{(sum)}_{t,k+1,i}}} \right] + \frac{2R\sqrt{k+1}\sqrt{1-\beta_2}}{\sqrt{1-\beta_1}}\bbE_{t-k-1}\left[\frac{g_{t-k,i}^2}{(\epsilon + v^{(sum)}_{t-k,i})}\right].
    \end{align*}
    Equivalently,
    \begin{align} \label{lemma3:C}
        \bbE [-|C|] \geq -\frac{1}{2}\bbE \left[\frac{G_{t-k,i}^2}{\sqrt{\epsilon + \tilde v^{(sum)}_{t,k+1,i}}} \right] - \frac{2R\sqrt{k+1}\sqrt{1-\beta_2}}{\sqrt{1-\beta_1}}\bbE_{t-k-1}\left[\frac{g_{t-k,i}^2}{(\epsilon + v^{(sum)}_{t-k,i})}\right]
    \end{align}
    Note,
    \begin{align*}
        \bbE [|A|] & \geq \sum_{i \in [d]}\sum_{k=0}^{t-1}\beta_1^k\left( \bbE \left[\frac{G_{t-k,i}^2}{\sqrt{\epsilon + \tilde v^{(sum)}_{t,k+1,i}}} \right] +\bbE [-|C|] \right)
    \end{align*}
    Then, inserting \eqref{lemma3:C}, we have:
    \begin{align}\label{lemma3:A}
        \bbE [|A|] & \geq \sum_{i \in [d]}\sum_{k=0}^{t-1}\beta_1^k\left( \bbE \left[\frac{G_{t-k,i}^2}{\sqrt{\epsilon + \tilde v^{(sum)}_{t,k+1,i}}} \right] - \left( \frac{1}{2}\bbE \left[\frac{G_{t-k,i}^2}{\sqrt{\epsilon + \tilde v^{(sum)}_{t,k+1,i}}} \right] + \frac{2R\sqrt{k+1}\sqrt{1-\beta_2}}{\sqrt{1-\beta_1}}\bbE_{t-k-1}\left[\frac{g^2}{(\epsilon + v^{(sum)}_{t-k,i})}\right]  \right) \right) \notag \\
        & = \frac{1}{2} \left( \sum_{i \in [d]}\sum_{k=0}^{t-1}\beta_1^k\bbE \left[\frac{G_{t-k,i}^2}{\sqrt{\epsilon + \tilde v^{(sum)}_{t,k+1,i}}} \right]\right)  - \frac{2R\sqrt{1-\beta_2}}{\sqrt{1-\beta_1}}\sum_{k=0}^{t-1}\beta_1^k\sqrt{k+1}\bbE [\|U_{t-k}\|^2].
    \end{align}
    Note, in \eqref{lemma3:first term} we have,
    \begin{align*}
        G_{t,i}\frac{m_{t,i}}{\sqrt{\epsilon + v^{(sum)}_{t,i}}} = A + B \geq A - |B|,
    \end{align*}
    taking the expectation of both sides yields
    \begin{align} \label{lemma3:intermediate}
        \bbE \left[ G_{t,i}\frac{m_{t,i}}{\sqrt{\epsilon + v^{(sum)}_{t,i}}} \right] = \bbE[A] + \bbE[B] \geq \bbE[A] + \bbE[-|B|],
    \end{align}
    Inserting \eqref{lemma3:A} and negated \eqref{lemma3:B} into \eqref{lemma3:intermediate} results in
    \begin{align*}
    \bbE \left[ G_{t,i} \frac{m_{t,i}}{\sqrt{\epsilon+v^{(sum)}_{t,i}}} \right] & \geq \frac{1}{2} \left( \sum_{i \in [d]}\sum_{k=0}^{t-1} \beta_1^k \bbE \left[ \frac{G^2_{t-k,i}}{\sqrt{\epsilon+\tilde v^{(sum)}_{t,k+1,i}}} \right] \right) - \frac{3R\sqrt{1-\beta_2}}{\sqrt{1-\beta_1}} \sum_{k=0}^{t-1} \sqrt{k+1} \beta_1^k \bbE \| U_{t-k} \|^2 \notag \\
    & \quad - \frac{\alpha^2L^2}{4R\sqrt{1-\beta_2}}\sqrt{1-\beta_1} \sum_{l=1}^{t-1}\| u_{t-l}\|^2 \sum_{k=l}^{t-1} \beta_1^k \sqrt{k},
    \end{align*}
    which completes the proof.
\end{proof}

\begin{proof}[Proof of Lemma~\ref{lemma4}]
    For some $t$,
    \begin{align*}
        \frac{c_t^2}{\epsilon+b_t} \leq \frac{1}{1-\beta_1}\sum_{\tau =1}^t \beta_1^{t-\tau} \frac{a_{\tau}^2}{\epsilon+b_t}.
    \end{align*}
    We have $\epsilon+b_t \geq \epsilon + b_\tau$ for any $\tau \leq t$, then,
    \begin{align*}
        \sum_{t=1}^T \frac{c_t^2}{\epsilon+b_t} & \leq \frac{1}{1-\beta_1}\sum_{t=1}^T \sum_{\tau =1}^t \beta_1^{t-\tau} \frac{a_{\tau}^2}{\epsilon+b_\tau} \\
        & = \frac{1}{1-\beta_1} \sum_{\tau = 1}^T \frac{a_{\tau}^2}{\epsilon+b_\tau} \sum_{\tau =1}^t \beta_1^{t-\tau} \\
        & \leq \frac{1}{(1-\beta_1)^2} \sum_{\tau = 1}^T \frac{a_{\tau}^2}{\epsilon+b_\tau} \\
        & \leq \frac{1}{(1-\beta_1)^2}\ln{\left( 1+\frac{b_T}{\epsilon}\right)},
    \end{align*}
    where in the last inequality we apply Lemma~\ref{thm3:lemma2} with $\rho = 0$. Note, from the definition of $b_T$ we also have that $b_T \leq \frac{R^2T}{1-\beta_2}$.
\end{proof}

\section{Additional Experiments and Details}\label{app:experiments}

\subsection{Overall Parameterization in the Experiments}

In Section~\ref{sec:experiments}, we provided the overall parameterization verbally, here we give the vector updates for the overall parameterization. For the \emph{first order moment term} we have, 
\begin{align*}
    m_t &=\beta_{1} m_{t-1} + (1-\beta_{1})g_t \\
    n_t &= \beta_{3}n_{t-1} + (1-\beta_3)(g_t-g_{t-1})\\
    m_t^{lion} &= \beta_2^{lion} m_{t-1}^{lion} + (1-\beta_2^{lion})g_t \\
    u_t &= \beta_1^{lion} m_{t-1}^{lion} + (1-\beta_1^{lion})g_t ,
\end{align*} 
where $\beta_1$ controls the strength of heavy-ball momentum and $\beta_3$ extends the equation to Nesterov momentum. The first order moment equations essentially capture the updates for Adan and Adam. For the \emph{second order moment term} we have,
\begin{align*}
    \hat g_t &= g_t + \beta_{3}(g_t - g_{t-1}) \\
    \tilde g_t^2 &= c_{t-1} \hat g_t^2 + (1-c_{t-1}) ( \bar v_{t-1}+\hat g^2_t \text{sign}(\hat g_t^2 - \bar v_{t-1})) \\
    \bar v_t &= \beta_{2} \bar v_{t-1} + (1-\beta_{2})\tilde g_t^2 \\
    \tilde v_t &= (\bar v_t + (t-1)\tilde v_{t-1})/t \\
    v_t &= \rho \bar v_t + (1-\rho) \tilde v_t
\end{align*}
where $c$ controls whether the difference of consequent $\bar v_t$s grows with $g_t^2\text{sign}(g_t^2-\bar v_{t-1})$, as in YOGI, or $g_t^2(g_t^2-\bar v_{t-1})$ as in Adam; the former, prevents rapid increases in effective learning rate and provides more controlled updates. For the second order moment term, $\rho_t$ determines whether to use, less noisy, average of past $v_t$s (we call this method AVGrad and defer its formal introduction to the next section) or the current $v_t$, which may be more noisy but up to date. Lastly, the \emph{update term} together with the decoupled weight decay is as follows,
\begin{align*}
    &x_{t-1} = x_{t-1} - \lambda \alpha_t x_{t-1} \\
     &x_t = x_{t-1} {-} \alpha_t \Big( \gamma \frac{m_t + \beta_{3}n_t}{\sqrt{v_t} + \epsilon} + (1-\gamma)\text{sign}(u_t) \Big), \notag
\end{align*}
where $\epsilon$ is the stability parameter, $\text{sign}(u_t)$ is the element-wise sign operator, $\gamma$ controls whether to do a Lion type update or not.

\subsection{Training setup for the experiments} 
On OpenWebText, we use a global batch size of 480 sequences, cosine learning rate schedule with the peak learning rate of $6 \times 10^{-4}$ ($1.5 \times 10^{-4}$ for Lion) and the final learning rate of $1.5 \times 10^{-5}$ (as chosen for Sophia algorithm in \citep{liu2023sophia}). \kaancr{ For Lion we use the tuned learning rate from (Liu, 2023) for the same setting. For Adam we found that employing Sophia's learning rate schedule results in better loss compared to employing Adam's learning rate schedule in \citep{liu2023sophia}. For our methods: AVGrad, \Mada, \texttt{MADA-FS}, we directly employ Adam's learning rate and learning schedule without any tuning since our goal is to demonstrate that \Mada can be plugged in place of Adam without any changes in learning rate schedule.}

We run the experiment for 100,000 iterations (first 2000 are warmup iterations) which corresponds to training on $\sim 492$B tokens. We use a weight decay parameter of $0.1$. On Shakespeare, we use a batch size of 64, cosine learning rate schedule with the peak learning rate of $10^{-3}$ and the final learning rate of $10^{-4}$. We run the experiment for 5,000 iterations (first 100 are warmup iterations). We run our experiments on AWS \texttt{p5.48xlarge} instances equipped with 8 NVIDIA H100 GPUs. To be able utilize multiple GPU's we allreduced hyper-gradients across GPU's.

We initialize \Mada with $\beta_3 = 0.9, \rho=0$ on top of Adam's established $\beta_1,\beta_2$ parameters for OpenWebText experiments as it resulted in a better validation loss. For Shakespeare experiments, we compare \Mada against Adam over a grid of $(\beta_1, \beta_2)$ values, where in the case of \Mada $(\beta_1, \beta_2)$ represents the initial values. For the OpenWebText experiments, we do not update hyper-parameters for the first 50 iterations for the sake of stability. For both experiments we use SGD as the hyper-parameter optimizer. On Shakespeare, for 10M model we use $2.5e{-}3$ learning rate and $0.5$ momentum (we do not use momentum for $\gamma$) for learning the hyper-parameters. On OpenWebText (and for 1.5B model on Shakespeare) we use a learning rate of $5e{-}4$ for training $\beta_1,\beta_2$ and $1e{-}1$ for other hyper-parameters.

\textbf{Vision Tasks.} For 5-layer model experiments, we use a CNN whose first two layers are convolutional layers with 6 and 16 output channels and 5 kernel size; and last 3 layers are fully connected layers with 120, 84, and 10 output dimensionality. We use ReLU activation in all layers except the last one and maxpool on the outputs convolutional layers. We use a constant learning rate of $1e{-}3$ for MADA, Adam, HyperAdam; and $1e{-}2$ learning rate and 0.9 momentum coefficient for SGD with momentum. We initialize \Mada from Adam state, we set hyper learning rate for $\beta_1, \beta_2$ to $1e{-}4$ and for the other variables to $1e{-}2$. We use batch size of 256 and train for 50 epochs. For ResNet-9 experiments, we use the model implementation from \url{https://github.com/Moddy2024/ResNet-9}. We use one cycle learning rate scheduler with $1e{-2}$ peak learning rate for \Mada, Adam, HyperAdam and $1e{-1}$ peak learning rate for SGD with momentum. Again we initialize \Mada from Adam, we set hyper learning rate for $\beta_1$ to $1e{-}4$, for $\beta_2$ to $1e{-}4$ and for the other variables to $1e{-}2$. We use batch size of 400 and train for 50 epochs.

\subsection{Additional Experiments}
\begin{figure}[h]
\centering
    \begin{tabular}{ccccc} \hline
  Method & OpenWebText (validation loss) & OpenWebText & Wikitext & Lambada \\ \hline
  \Mada  & 2.8853 & 17.9084 & 61.4689 & 73.1291 \\ 
  \Mada-FS  & 2.8838 & 17.8822 & 63.9886 &  75.6158\\ \hline
  \end{tabular}
  \captionof{table}
      {%
        Validation loss on OpenWebText and validation perplexities on OpenWebText, Wikitext and Lambada datasets of GPT-2 (125M) models trained on OpenWebText with \Mada when the initial optimizer state is AVGrad.%
      }
\end{figure}

\begin{figure}[h]
\centering
    \begin{tabular}{ccccc} \hline
  Method & OpenWebText (validation loss) & OpenWebText & Wikitext & Lambada \\ \hline
  Adam  & 2.6527 & 14.1928 & 50.1410 & 53.7468 \\ 
  \Mada & 2.6422 & 14.0441 & 43.8067 &  53.7575\\ \hline
  \end{tabular}
  \captionof{table}
      {%
        Validation loss on OpenWebText and validation perplexities on OpenWebText, Wikitext and Lambada datasets of GPT-2 (355M) models trained on OpenWebText with Adam and \Mada. The initial state of \Mada is AVGrad + Adan ($\rho=0,\beta_3 = 0.9$).%
      }
\end{figure}

In these tables, we observe \Mada outperforms Adam in training of GPT-2 (medium) as well. Moreover, we observe GPT-2 (small) training with different initialization (from AVGrad).\\

\textbf{Synthetic convex experiment.} This is a famous example from \citep{reddi2018}, where Adam notably fails to converge to the optimal solution, $x=-1$; whereas AMSGrad (and AVGrad) reach the optimum. The online learning experiment given in \citep{reddi2018} to motivate AMSGrad is as follows:
\begin{align}
g_t(x) := \begin{cases}        1010x & \text{for } t \text{ mod } 101=1 \\       -10x & \text{otherwise}     \end{cases} 
\end{align}
with constraint set $x \in [-1,1]$ and $t$ denotes the time index in the online learning setting. In Figure~\ref{fig:amsgradrho}, we plot the average regret which is defined by $(g_t(x)-g_t(-1))/t$, as well as the evolution of $x$ and $\rho$ with respect to iterations. Here, we reduce the effect of Lion by assigning a small hyper-learning rate to $\gamma$, since signSGD based methods neutralize the effect of large gradients which allow faster progress towards the optimum. We observe that even when we initialize \Mada from Adam ($\rho_0 = 1$), it quickly recovers AVGrad ($\rho_t \rightarrow 0$), which is the right optimizer to use for this example. This experiment shows that \Mada can learn to behave like AVGrad even when initialized from Adam.
\begin{figure}[h]
    \centering
    \includegraphics[scale = 0.35]{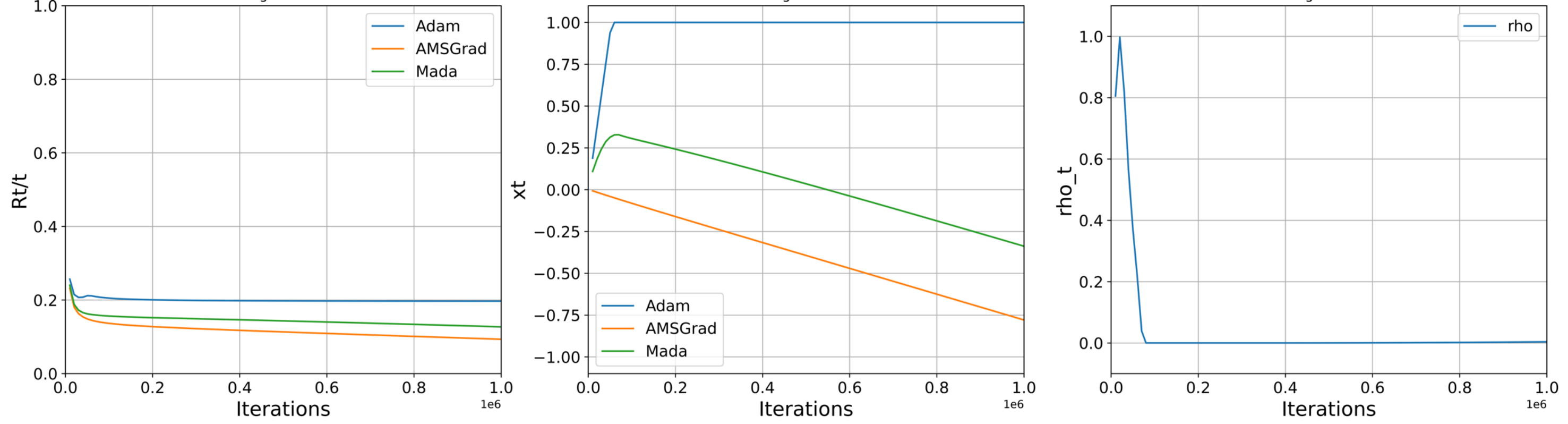}
    \caption{Average regret, $x_t$, $\rho_t$ with respect to iterations for \Mada on simple convex function.}
    \label{fig:amsgradrho}
\end{figure}

\section{Details on Computational Burden} \label{app:computation}

As mentioned in the main text, the additional computation can be broken down into two components: additional per-parameter computational steps during optimizer update, and the computation of hyper-gradients. Note that the first component is negligible compared to the overall FLOPs requirement of a single training step in a language model. This is because for a model with $N$ parameters trained over a batch of $T$ tokens, the model parameter update in the parameterized optimizer involves $cN$ FLOPs (with $c$ being on the range of 10-20 depending on the parameterization), while the forward-backward passes require approximately $6TN$ FLOPs, with $T$ being on the order of thousands. To analyze the second component, consider the hyper-gradient example in \eqref{eq:hypergrad}, where the computation of hyper-gradient with respect to $\rho$ involves several vector-level element-wise operations (note that the multiplication of the first term with second does not require a matrix-vector multiplication, since it can be implemented by a dot product with the numerator followed by element-wise division with the denominator), where each vector is of size $N$. As a result, the FLOPs requirement is still $O(N)$ (does not scale with $T$), and much smaller than the $6TN$ required for the forward-backward passes. The derivative with respect to the other coefficients can similarly be shown to have $O(N)$ complexity. \\
We also profile the memory footprints of the methods and find that the peak memory usage is 15.5 GB for Adam, and 24.9 GB for \Mada; time per iteration is 0.65s for Adam and 0.85s for \Mada. If we exclude LION (which contributes the least in this setting) MADA results in 22.2GB of memory usage and 0.72s time per iteration.

\end{document}